\documentclass[final]{MIT8X10M}
\slugtoks{Hazan, Papandreou, Tarlow: Perturbations, Optimization, and Statistics}

\usepackage{fleqn}    
\usepackage{booktabs} 
\usepackage{multirow} 

\usepackage{subfig}
\usepackage{subfloat}
\newcounter{subfigure@save}[figure]  

\usepackage{color}
\usepackage{ifthen}
\usepackage{eucal}

\usepackage{graphicx}

\usepackage[sectionbib]{chapterbib_numbered_bibliography}
\usepackage[sectionbib,round]{natbib}

\usepackage[fleqn]{amsmath}
\usepackage{amssymb}
\usepackage{amsbsy}

\usepackage[chapter]{algorithm}
\usepackage{algorithmicx}
\usepackage{algpseudocode} 
\usepackage[ruled,vlined]{algorithm2e} 

\let\proof\relax

\usepackage{amsthm}
\newtheorem{theorem}{Theorem}[chapter]

\newtheorem{proposition}[theorem]{Proposition}
\newtheorem{corollary}[theorem]{Corollary}

\usepackage{url}

\usepackage{smallcaptions}

\usepackage{bm}

\usepackage{pdfpages}

\usepackage{titletoc}

\newcommand{\qc}{\mathcal{q}}

\DeclareMathOperator*{\argmax}{argmax}

\newcommand{\Algo}[1]{Algorithm~\protect\ref{#1}}

\newcommand{\Fig}[1]{Figure~\protect\ref{#1}}

\newcommand{\tab}[1]{Table~\protect\ref{#1}}

\newcommand{\Sec}[1]{Section~\protect\ref{#1}}

\mathchardef\ordinarycolon\mathcode`\:
\mathcode`\:=\string"8000
\begingroup \catcode`\:=\active
  \gdef:{\mathrel{\mathop\ordinarycolon}}
\endgroup

\newenvironment{abstract}
 {\begin{list}{}
    {\setlength{\leftmargin}{0em}
     \setlength{\rightmargin}{0em}
     \setlength{\itemsep}{0pt}
     \setlength{\topsep}{0pt}} 
    \item[] \em }
 {\end{list}\par\vspace{8ex}\par}

\newenvironment{authors}
 {\vspace{-4ex}\begin{list}{}
    {\setlength{\leftmargin}{-.5\marginpartotal}
     \setlength{\rightmargin}{0pt}
     \setlength{\itemsep}{0pt}
     \setlength{\topsep}{0pt}} 
    \item[] }
 {\end{list}\par\vspace{8ex}\par}

\def\AUname#1{\par\makebox[2.8in][l]{\bf #1}}
\def\AUemail#1{{\tt #1}}
\def\AUweb#1{}
\def\AUaffiliation#1{\\\emph{#1}}
\def\AUaddress#1{\\\emph{#1}\par\medskip}

\makeatletter
\renewcommand{\small}{\@setfontsize\normalsize\@xpt\@xipt}
\renewcommand\appendix[1]{\par%
  \addtocounter{section}{1}%
  \setcounter{subsection}{0}%
  \section*{Appendix: #1}%
}
\makeatother

\numberwithin{equation}{chapter}

\begin{document}

\raggedbottom

%
%
%
%
%
%
%
%
%
%
%




\newcommand{\dataset}{{\cal D}}
\newcommand{\fracpartial}[2]{\frac{\partial #1}{\partial  #2}}
\newcommand{\bx}{{\bf x}}
\newcommand{\hx}{\hat{\bf x}}
\newcommand{\hv}{\hat{\bf v}}
\newcommand{\hpi}{\hat{\pi}}
\newcommand{\hphi}{\hat{\phi}}
\newcommand{\hta}{\hat{\theta}}
\newcommand{\tx}{\tilde{\bf x}}
\newcommand{\by}{{\bf y}}
\newcommand{\bw}{{\bf w}}
\newcommand{\hw}{{\bf \hat{w}}}
\newcommand{\bff}{{\bf f}}
\newcommand{\bgg}{{\bf g}}
\newcommand{\bv}{{\bf v}}
\newcommand{\bp}{{\bf p}}
\newcommand{\tp}{\tilde{p}}
\newcommand{\bu}{{\bf u}}
\newcommand{\bs}{{\bf s}}
\newcommand{\btt}{{\bf t}}
\newcommand{\bz}{{\bf z}}
\newcommand{\bhb}{{\bf \hbar}}
\newcommand{\bhh}{{\bf \hat{h}}}
\newcommand{\tih}{{\tilde h}}
\newcommand{\bd}{{\bf d}}
\newcommand{\bfe}{{\bf e}}
\newcommand{\bn}{{\bf n}}
\newcommand{\bE}{{\bf E}}
\newcommand{\bM}{{\bf M}}
\newcommand{\bN}{{\bf N}}
\newcommand{\bW}{{\bf W}}
\newcommand{\hW}{{\bf \hat{W}}}
\newcommand{\hV}{{\bf \hat{V}}}
\newcommand{\hp}{{\tilde{p}}}
\newcommand{\pcd}{{p_{\tiny\bf CD}}}
\newcommand{\bP}{{\bf P}}
\newcommand{\bU}{{\bf U}}
\newcommand{\bV}{{\bf V}}
\newcommand{\bH}{{\bf H}}
\newcommand{\bS}{{\bf S}}
\newcommand{\bX}{{\bf X}}
\newcommand{\bL}{{\bf L}}
\newcommand{\baa}{{\bf a}}
\newcommand{\bbb}{{\bf b}}
\newcommand{\bee}{{\bf e}}
\newcommand{\bi}{{\bf i}}
\newcommand{\bj}{{\bf j}}
\newcommand{\bo}{{\bf o}}
\newcommand{\bzero}{{\bf 0}}
\newcommand{\bcc}{{\bf c}}
\newcommand{\bI}{{\bf I}}
\newcommand{\bG}{{\bf G}}
\newcommand{\bJ}{{\bf J}}
\newcommand{\bA}{{\bf A}}
\newcommand{\bB}{{\bf B}}
\newcommand{\bC}{{\bf C}}
\newcommand{\bK}{{\bf K}}
\newcommand{\bq}{{\bf q}}
\newcommand{\br}{{\bf r}}
\newcommand{\bxi}{\boldsymbol{\xi}}
\newcommand{\bze}{\boldsymbol{\zeta}}
\newcommand{\bnu}{\boldsymbol{\nu}}
\newcommand{\boeta}{\boldsymbol{\eta}}
\newcommand{\bR}{{\bf R}}
\newcommand{\bQ}{{\bf Q}}
\newcommand{\bD}{{\bf D}}
\newcommand{\bg}{{\bf g}}
\newcommand{\bGam}{{\bf \Gamma}}
\newcommand{\bbt}{\boldsymbol{\beta}}
\newcommand{\bLa}{{\bf \Lambda}}
\newcommand{\bOm}{{\bf \Omega}}
\newcommand{\Om}{{\Omega}}
\newcommand{\bpsi}{\boldsymbol{\psi}}
\newcommand{\bF}{{\bf F}}
\newcommand{\eE}{\mathbb{E}}
\newcommand{\eI}{\mathbb{I}}
\newcommand{\eV}{\mathbb{V}}
\newcommand{\eG}{\mathbb{G}}
\newcommand{\eP}{\mathbb{P}}
\newcommand{\eS}{\mathbb{S}}
\newcommand{\bmu}{\boldsymbol{\mu}}
\newcommand{\bal}{\boldsymbol{\alpha}}
\newcommand{\btau}{\boldsymbol{\tau}}
\newcommand{\brho}{\boldsymbol{\rho}}
\newcommand{\bbet}{\boldsymbol{\beta}}
\newcommand{\bgam}{\boldsymbol{\gamma}}
\newcommand{\bka}{\boldsymbol{\kappa}}
\newcommand{\bde}{\boldsymbol{\delta}}
\newcommand{\bPhi}{\boldsymbol{\Phi}}
\newcommand{\bzeta}{\boldsymbol{\zeta}}
\newcommand{\bb}{{\bf b}}
\newcommand{\bsg}{\boldsymbol{\sigma}}
\newcommand{\bSig}{{\bf \Sigma}}
\newcommand{\Sig}{{\Sigma}}
\newcommand{\eps}{\varepsilon}
\newcommand{\Lam}{{\Lambda}}
\newcommand{\bLam}{{\bf \Lambda}}
\newcommand{\al}{\alpha}
\newcommand{\ga}{\gamma}
\newcommand{\bga}{\boldsymbol{\gamma}}
\newcommand{\la}{\lambda}
\newcommand{\bla}{\boldsymbol{\lambda}}
\newcommand{\bta}{\boldsymbol{\theta}}
\newcommand{\bphi}{\boldsymbol{\phi}}
\newcommand{\bvphi}{\boldsymbol{\varphi}}
\newcommand{\pol}{\boldsymbol{\varphi}}
\newcommand{\bpi}{\boldsymbol{\pi}}
\newcommand{\ta}{\theta}
\newcommand{\sg}{\sigma}
\newcommand{\ka}{\kappa}
\newcommand{\beps}{\boldsymbol{\varepsilon}}
\newcommand{\law}{\leftarrow}
\newcommand{\ra}{\rightarrow}
\newcommand{\Ra}{\Rightarrow}
\newcommand{\de}{\delta}
\newcommand{\bt}{\beta}
\newcommand{\ha}{\frac{1}{2}}
\newcommand{\bone}{{\bf 1}}
\newcommand{\cC}{{\cal C}}
\newcommand{\cF}{{\cal F}}
\newcommand{\cG}{{\cal G}}
\newcommand{\cA}{{\cal A}}
\newcommand{\cM}{{\cal M}}
\newcommand{\cJ}{{\cal J}}
\newcommand{\cN}{{\cal N}}
\newcommand{\cP}{{\cal P}}
\newcommand{\cD}{{\cal D}}
\newcommand{\cE}{{\cal E}}
\newcommand{\cL}{{\cal L}}
\newcommand{\cT}{{\cal T}}
\newcommand{\cX}{{\cal X}}
\newcommand{\cY}{{\cal Y}}
\newcommand{\cB}{{\cal B}}
\newcommand{\cS}{{\cal S}}
\newcommand{\cQ}{{\cal Q}}
\newcommand{\cH}{{\cal H}}
\newcommand{\cZ}{{\cal Z}}
\newcommand{\cR}{{\cal R}}
\newcommand{\cO}{{\cal O}}
\newcommand{\cW}{{\cal W}}
\newcommand{\cU}{{\cal U}}
\newcommand{\cV}{{\cal V}}
\newcommand{\cch}{{\cal c}}
\newcommand{\cNIW}{{\cal NIW}}
\newcommand{\tA}{{\tilde A}}
\newcommand{\tS}{{\tilde S}}
\newcommand{\tpi}{{\tilde \pi}}
\newcommand{\tbt}{{\tilde \bt}}
\newcommand{\tmu}{{\tilde \mu}}
\newcommand{\tta}{{\tilde \theta}}
\newcommand{\tit}{{\tilde t}}
\newcommand{\tf}{{\tilde f}}
\newcommand{\tsg}{{\tilde \sigma}}
\newcommand{\tP}{{\tilde P}}
\newcommand{\bwp}{{\bf\wp}}
\newcommand{\hN}{{\hat{N}}}
\newcommand{\barN}{{\bar{N}}}
\newcommand{\blamp}{{\bla^{\tt MP}}}
\newcommand{\map}{{\bf\tiny MAP}}
\newcommand{\nN}{\frac{1}{N}\sum_{n=1}^N}
\newcommand{\sN}{\sum_{n=1}^N}
\newcommand{\sK}{\sum_{i=1}^K}
\newcommand{\sjK}{\sum_{j=1}^K}
\newcommand{\tr}{{\bf tr}}
\newcommand{\CF}{{\bf C\!F}}
\newcommand{\bra}{\left\langle}
\newcommand{\ket}{\right\rangle}
\newcommand{\dd}{\partial}
\newcommand{\nn}{\nonumber}
\newcommand{\grad}{\nabla}
\newcommand{\KL}{\mbox{KL}}
\newcommand{\bks}{\backslash}
\newcommand{\equil}{{\mbox{\tiny EQ}}}
\newcommand{\bethe}{{\mathtt{BP}}}
\newcommand{\tDP}{{\text{DP}}}
\newcommand{\td}{\text{d}}
\newcommand{\Pa}{{\mbox{pa}}}
\newcommand{\Ch}{{\mbox{ch}}}
\newcommand{\Sd}{{\Sig_\text{dyn}}}
\newcommand{\Su}{\Sig_\text{u}}
\newcommand{\Sg}{\Sig_\text{g}}
\newcommand{\Ska}{\Sig_{\ka}}
\newcommand{\Cw}{C_\text{wall}}
\newcommand{\tR}{\tilde{R}}
\newcommand{\hP}{{\hat{P}}}
\newcommand{\hJ}{{\hat{J}}}
\newcommand{\hQ}{{\hat{Q}}}

\newcommand{\be}{\begin{equation}}
\newcommand{\ee}{\end{equation}}

\newcommand{\defeq}{\overset{\textnormal{def}}=}


\chapter[Herding as a Learning System with Edge-of-Chaos Dynamics\\%
{\normalsize\rm\emph{Y.~Chen, and M.~Welling}}]%
{Herding as a Learning System with Edge-of-Chaos Dynamics}

\label{chapter:welling}

\markboth{Herding as a Learning System with Edge-of-Chaos Dynamics}{}

\begin{authors}
\AUname{Yutian Chen}
\AUemail{yutianc@google.com}
\AUaffiliation{Google DeepMind}
\AUaddress{London, UK}
\AUname{Max Welling}
\AUemail{m.welling@uva.nl}
\AUaffiliation{University of Amsterdam}
\AUaddress{Amsterdam, Netherlands}
\end{authors}


\begin{abstract}
Herding defines a deterministic dynamical system at the edge of chaos. 
It generates a sequence of model states and parameters by alternating parameter perturbations with state maximizations,
where the sequence of states can be interpreted as ``samples'' from an associated MRF model.  
Herding differs from maximum likelihood estimation in that the sequence of parameters does not converge to a fixed point 
and differs from an MCMC posterior sampling approach in that the sequence of states is generated deterministically. 
Herding may be interpreted as  a``perturb and map" method where the parameter perturbations are generated using a deterministic nonlinear dynamical system
rather than randomly from a Gumbel distribution. This chapter studies the distinct statistical
characteristics of the herding algorithm and shows that the fast convergence rate of the controlled moments may 
be attributed to edge of chaos dynamics. The herding algorithm can also be generalized to
models with latent variables and to a discriminative learning setting. The perceptron cycling theorem
ensures that the fast moment matching property is preserved in the more general framework.
\end{abstract}

\section{Introduction}
The traditional view of a learning system is one where an initial parameter vector $\bw_0$ is updated until some convergence criterion is met: $\bw_0,\bw_1,..,\bw_T$ with (in theory) $T\ra\infty$ and $\bw_\infty=\bw^*$ a fixed point of the updates. These updates usually maximize some objective such as the log-likelihood of the data. We can view this process as a dynamical system with a contractive map $\bw_{t+1}=F_t(\bw_t)$ which is designed to iterate to a fixed point. The map $F_t$ can be either deterministic or stochastic. For instance, batch gradient descent is an example of a deterministic map while stochastic gradient descent is an example of a stochastic map. A natural question is whether the existence of a fixed point $\bw^*$ is important, and whether meaningful learning systems can exist that do not converge to any fixed point but traverse an attractor set. To answer this question we can draw inspiration from Markov chain Monte Carlo (MCMC) procedures which generate samples from a posterior distribution $P(\bw|\cD)$ (with $\cD$ indicating the data). MCMC also generates a sequence of parameter values $\bw_0,..,\bw_T$ but one that does not converge to a fixed point. Rather the samples form an attractor set with a measure (density) equal to the posterior distribution. One can make meaningful predictions with MCMC chains by making predictions for every sampled model $\bw_t$ separately and subsequently averaging the predictions. There is also evidence that learning in the brain is a dynamical process. For instance, \cite{aihara1982temporally} have described chaotic dynamics in the Hodgkin-Huxley equations for membrane dynamics and studied them experimentally in squid giant axons. Also, much evidence has now been accumulated that synapses are subject to fast dynamical processes such as postsynaptic depression and facilitation \citep{TsodyksPawelzikMarkram98}.

Herding \citep{Welling09A} is perhaps the first learning dynamical system based on a deterministic map and with a nontrivial attractor (i.e. not a single fixed point). It emerged from taking the limit of infinite stepsize in the usual (maximum likelihood) updates for a Markov random field (MRF) model. It can be observed that in this limit the parameters will not converge to a fixed point but rather traverse a usually non-periodic trajectory in weight space. The information contained in the data is now stored in the trajectories (or the attractor) of this dynamical system, rather than in a point estimate of a collection of parameters. In fact it can be shown that this dynamical system is neither periodic (under some conditions) nor chaotic, a state which is associated with ``edge of chaos" dynamics. As illustrated in this chapter, by slowly increasing the stepsize (or equivalently lowering the temperature) we will move from a standard MRF maximum likelihood learning system with a single fixed point, through a series of period doublings to a system on the edge of chaos. One can show that the attractor is sometimes fractal, and that the Lyapunov exponents of this system are equal to $0$ implying that two nearby trajectories will eventually separate but only polynomially fast (and not exponentially fast as with chaotic systems). Many of the dynamical properties of this system are described by the theory of ``piecewise isometries" \citep{goetz2000dynamics}.  

Herding can thus be viewed as a dynamical system that generates state-space samples $\bs_1,..,\bs_T$ that are highly similar to the samples that would be generated by a learned MRF model with the same features. The state-space samples satisfy the usual moment matching constraints that defines an MRF and can be used for making meaningful predictions. In a way, herding combines learning and inference in one dynamical system. However, the distribution from which herding generates samples is not identical to the associated MRF because while the same moment matching constraints are satisfied, the entropy of the herding samples is usually somewhat lower than the (maximal) entropy of the MRF. The sequence of samples in state space $\bs_1,..,\bs_T$ has very interesting properties. First, it forms an infinite memory sequence as every sample depends on all the previous samples and not just the most recent sample as in Markov sequences. It can be shown that the number of distinct subsequences of length $T$ grows as $\cO(\log(T))$ implying that their (topological) entropy vanishes. For simple systems these sequences can be identified with ``low discrepancy sequences'' and Sturmian sequences \citep{MorseHedlund40}. Probably related to this is the fact that Monte Carlo averages based on these sequences converge as $\cO(1/T)$. This should be contrasted with random independent samples from the associated MRF distribution for which the convergence follows the usual $\cO(1/\sqrt{T})$ rate. Herding sequences thus exhibit strong negative auto-correlations leading to the faster convergence of Monte Carlo averages. It is conjectured that this property is related to the edge of chaos characterization of herding, and that both stochastic systems (such as samplers) as well as fully chaotic systems will always generate samples that can at most result in $\cO(1/\sqrt{T})$ convergence of Monte Carlo averages. 

Similar to ``perturb and map'' \citep{PaYu11a}, the execution of the herding map requires one to compute the maximum a posteriori (MAP) state defined by the current parameter setting. While maximization is sometimes easier than computing the expectations required to update the parameters of an MRF, for complex models maximization can also be NP hard. A natural question is therefore if one can relax the requirement of finding the MAP state and get away with partial maximization to, say, a local maximum instead of the global maximum. The answer to this question comes from a theorem that was proven a long time ago in the context of Rosenblatt's perceptron \citep{rosenblatt1958perceptron} and is known as the ``perceptron cycling theorem" (PCT) \citep{minsky1969perceptrons}. This theorem states precisely which conditions need to be fulfilled by herding at every iteration in order for the algorithm to satisfy the moment constraints. The PCT therefore allows us to relax the condition of finding the MAP state at every iteration, and as a side effect also allows us to run herding in an online setting or with stochastic minibatches instead of the entire dataset. A further relaxation of the herding conditions was described in \cite{chen2014herdingbookchapter} where it was shown that herding with \emph{inconsistent} moments as input (moments that can not be generated by a single joint probability distribution) still makes sense and generates the Euclidean projections of these moments on the marginal polytope. 

Like MRF models can be extended to models with hidden variables and to discriminative models such as the conditional Markov random field (CRF) models, herding can also be generalized along these same dimensions. Herding with hidden variables was described in \cite{Welling09B} and shown to increase the ability of this dynamical system to represent complex dependencies. Conditional herding was described in \cite{GelfandMaatenChenWelling10} and shown to be equivalent to the voted perceptron algorithm \cite{freund1999large} and to Collins' ``voted HMM" \cite{collins2002discriminative} in certain special cases. The herding view allowed the extension of these discriminative models to include hidden variables.

Herding is related to (or has been connected to) a number of optimization, learning and inference methods. Herding has obvious similarities to the concept of ``fast weights'' introduced by \cite{TielemanHinton09}. Fast weights follow a dynamics that is designed to make the Markov chain embedded in a MRF learning process mix fast. A similar idea was used in \cite{breuleux2011quickly} to speed up the mixing rate of an (approximate) sampling procedure. By applying herding dynamics conditionally w.r.t. its parent-states for every variable in a graphical model yet another fast mixing sampling algorithm was developed, called ``herded Gibbs'' \cite{bornn2013herded}. Herding was extended in \cite{ChenSmolaWelling10} to a deterministic sampling algorithm in continuous state spaces (known as ``kernel herding''). The view espoused in that paper led to an analysis of herding as a conditional gradient optimization algorithm (or Franke-Wolfe algorithm) in \cite{Bach2012herding} from which an improved convergence analysis emerged as well generalizations to versions of herding with non-uniform weights. In related work of \cite{Huszar12} it was shown that an optimally weighted version of (kernel) herding is equivalent to Bayesian quadrature, again resulting in faster convergence. \cite{harvey2014near} focused on the convergence rate of herding with respect to the dimensionality of the feature vector and proposed a new algorithm that scaled near-optimally with the dimensionality.

Perhaps the method closest related to herding is ``perturb and map" estimation, where the parameters of a MRF model are perturbed by sampling from a Gumbel distribution followed by maximization over the states. Like in herded Gibbs, the procedure is only ``exact'' if exponentially many parameters are perturbed. Herding is however different from perturb and map in that the perturbations are generated sequentially and deterministically. 

This chapter is built on the results reported earlier in a series of conference papers \cite{Welling09A,Welling09B,WellingChen10,ChenSmolaWelling10,GelfandMaatenChenWelling10}. Our current understanding of herding is far from comprehensive but rather represents a first attempt to connect learning systems with the theory of nonlinear dynamical systems and chaos. We believe that it opens the door to many new directions of research with potentially surprising and exciting discoveries.

The chapter is organized as follows. In \Sec{sec:property} we introduce the herding algorithm and study its statistical property as both a learning algorithm and a dynamical system. In \Sec{sec:extensions} we provide a general condition for herding to satisfy the fast moment matching properties, under which the algorithm is extended for partially observed models and discriminative models. We evaluate the performance of the introduced algorithms empirically in \Sec{sec:experiments}. The chapter is concluded with a summary in \Sec{sec:summary} and a conclusion in \Sec{sec:conclusion}.

\section{Herding Model Parameters}\label{sec:property}

\subsection{The Maximum Entropy Problem and Markov Random Fields}

Define $\bx\in \mathcal{X}$ to be a random variable in the domain $\mathcal{X}$, and $\bphi=\{\phi_\al(\bx)\}$ to be a set of feature functions of $\bx$, indexed by $\al$. In the maximum entropy problem (MaxEnt), given a data set of $D$ observations $\mathcal{D}=\{\bx_i\}_{i=1}^D$, we want to learn a probability distribution over $\bx$, $P(\bx)$, such that the expected features, a.k.a. moments, match the average value observed in the data set, denoted by $\bar{\phi}_\al$. For the remaining degrees of freedom in the distribution we assume maximum ignorance which is expressed as maximum entropy. Mathematically, the problem is to find a distribution $P$ such that:
\be
P = \arg\max_{\mathcal{P}} \mathcal{H}(\mathcal{P})\quad\textrm{s.t.}~\mathbb{E}_{\bx\sim\mathcal{P}}[\phi_\al(\bx)]=\bar{\phi}_\al,~\forall \al \label{eqn:maxent}
\ee

The dual form of the MaxEnt problem is known to be equivalent to finding the maximum likelihood estimate (MLE) of the parameters $\bw=\{w_\al\}$ of a Markov Random Field (MRF) defined on $\bx$, each parameter associated with one feature $\phi_\al$:
\begin{align}
& \bw_{\textrm{MLE}} = \arg\max_{\bw} P(\mathcal{D}; \bw) = \arg\max_{\bw} \prod_{i=1}^D P(\bx_i; \bw), \\
& P(\bx; \bw) = \frac{1}{Z(\bw)} \exp\left(\sum_\al w_\al \phi_\al(\bx)\right), \label{eqn:mle-mrf}
\end{align}
where the normalization term $Z(\bw)=\sum_\bx\exp(\sum_\al w_\al \phi_\al(\bx))$ is also called the partition function. The parameters $\{w_\al\}$ act as Lagrange multipliers to enforce the constraints in the primal form \ref{eqn:maxent}. Since they assign different weights to the features in the dual form, we will also called them ``weights" below.

It is generally intractable to obtain the MLE of parameters because the partition function involves computing the sum of potentially exponentially many states. Take the gradient descent optimization algorithm for example. Denote the average log-likelihood per data item by 
\be
\ell(\bw) \defeq \frac{1}{D}\sum_{i=1}^D\log P(\bx_i;\bw)=\bw^T \bar\bphi - \log Z(\bw) \label{eqn:llh}
\ee
The gradient descent algorithm searches for the maximum of $\ell$ with the following update step:
\be
\bw_{t+1} = \bw_t + \eta (\bar{\bphi} - \mathbb{E}_{\bx\sim P(\bx;\bw_t)}[\bphi(\bx)]) \label{eqn:gd}
\ee
Notice however that the second term in the gradient that averages over the model distribution, $\mathbb{E}_{P(\bx;\bw)}[\bphi(\bx)]$, is derived from the partition function and cannot be computed efficiently in general. A common solution is to approximate that quantity by drawing samples using Markov chain Monte Carlo (MCMC) at each gradient descent step. However, MCMC is known to suffer from slow mixing when the state distribution has multiple modes or variables are strongly correlated \citep{neal1993probabilistic}. Furthermore, we can usually afford to run MCMC for only a few iterations in the nested loop for the sake of efficiency \citep{Neal92,Tieleman08}, which makes it even harder to obtain an accurate estimate of the gradient.

Even when the MRF is well trained, it is usually difficult to apply the model to regular tasks such as inference, density estimation, and model selection, because all of those tasks require the computation of the partition function. One has to once more resort to running MCMC or other approximate inference methods during the prediction phase to obtain an approximation.

Is there a method to speed up the inference step that exists in both the training and test phases? The herding algorithm was proposed to address the slow mixing problem of MCMC and combine the execution of MCMC in both training and prediction phases into a single process.

\subsection{Learning MRFs with Herding}

When there exist multiple local modes in a model distribution, an MCMC sampler is prone to getting stuck in local modes and it becomes difficult to explore the state space efficiently. However, that is not a serious issue at the beginning of the MRF learning procedure as observed by, for example, \cite{TielemanHinton09}. This is because the parameters keep being updated with a large learning rate $\eta$ at the beginning. Specifically, when the expected feature vector is approximated by a set of samples $\mathbb{E}_{P(\bx;\bw)}[\bphi(\bx)]\approx \frac{1}{M} \sum_{m=1}^M \bphi(\bx_m)$ in the MCMC approach, after each update in Equation \ref{eqn:gd}, the parameter $\bw$ is translated along the direction that tends to reduce the inner product of $\bw^T\bphi(\bx_m)$, and thereby reduces the state probability around the region of the current samples. This change in the state distribution helps the MCMC sampler escape local optima and mix faster.

This observation suggests that we can speed up the MCMC algorithm by updating the target distribution itself with a large learning rate. However, in order to converge to a point estimate of a model, $\eta$ needs to be decreased using some suitable annealing schedule. But one may ask if we are necessarily interested in a fixed value for the model parameters? As discussed in the previous subsection, for many applications one needs to compute averages over the (converged) model which are intractable anyway. In that case, a sequence of samples to approximate the averages is all we need. It then becomes a waste of resources and time to nail down a single point estimate of the parameters by decreasing $\eta$ when a sequence of samples is already available. We will actually kill two birds with one stone by obtaining samples during the training phase and reuse them for making predictions. The idea of the herding algorithm originates from this observation.

The herding algorithm proposed in \cite{Welling09A} can be considered as an algorithm that runs a gradient descent algorithm with a constant learning rate on an MRF in the zero-temperature limit. Define the distribution of an MRF with a temperature by replacing $\bw$ with $\bw / T$, where $T$ is an artificial temperature variable. The log-likelihood of a model (multiplied by $T$) then becomes:
\be
\ell_T(\bw) = \bw^T \bar\bphi - T\log\left(\sum_\bx\exp\left(\sum_\al \frac{w_\al}{T} \phi_\al(\bx)\right)\right) \label{eqn:llh_T}
\ee

When $T$ approaches $0$, all the probability is absorbed into the most probable state, denoted as $\bs$, and the expectation of the feature vector, $\bar\bphi$, equals that of state $\bs$. The herding algorithm then consists of the iterative gradient descent updates in the limit, $T\rightarrow 0$, with a constant learning rate, $\eta$:
\begin{align}
\bs_t &= \arg\max_\bx \sum_\al w_{\al,t-1} \phi_\al (\bx) \label{eqn:herding_1} \\
\bw_{t} &= \bw_{t-1} + \eta (\bar\bphi - \bphi(\bs_t)) \label{eqn:herding_2}
\end{align}
We usually set $\eta=1$ except when mentioned explicitly because the herding dynamics is invariant to the learning rate as explained in \Sec{sec:basic_properties}. We treat the sequence of most probable states, $\{\bs_t\}$, as a set of ``samples" for herding and use it for inference tasks. At each iteration, we find the most probable state in the current model distribution deterministically, and update the parameter towards the average feature vector from the training data subtracted by the feature vector of the current sample. Compared to maintaining a set of random samples in the MCMC approach \citep[see e.g.][]{Tieleman08}, updating $\bw$ with a single sample state facilitates updating the distribution at an even rate. 

If we divide both sides of Equation \ref{eqn:herding_2} by $T$ and redefine $\frac{\bw}{T} \rightarrow \bw'$ in both Equations \ref{eqn:herding_1}-\ref{eqn:herding_2}, 
\be
\frac{\bw_{t+1}}{T} = \frac{\bw_t}{T} + \frac{\eta}{T} (\bar{\bphi} - \mathbb{E}_{\bx\sim P(\bx;\frac{\bw_t}{T})}[\bphi(\bx)]) \label{eqn:gd_T}
\ee
we see that, after taking the limit $T\rightarrow\infty$, we can interpret herding as maximum likelihood learning with infinitely large stepsize and rescaled weights. The surprising observation is that the state sequence $\{\bs_t\}$ generated by this process is still meaningful and can be interpreted as approximate samples from an MRF model with the correct moment constraints on the features $\bphi(\bx)$. 

One can obtain an intuitive impression of the dynamics of herding by looking at the change in the asymptotic behavior of the gradient descent algorithm as we decrease $T$ in Equation \ref{eqn:gd_T} from a large value towards $0$. Assume that we can compute the expected feature vector w.r.t. the model exactly. Given an initial value of $\bw$, the gradient descent update equation \ref{eqn:gd_T} with a constant learning rate is a deterministic mapping in the parameter space. When $T$ is large enough ($\eta/T$ is small enough), the optimization process will converge and $\bw/T$ will approach a single point which is the MLE. As $T$ decreases below some threshold ($\eta/T$ is above some threshold), the convergence condition is violated and the trajectory of $\bw_t$ will move asymptotically into an oscillation between two points, that is, the attractor set splits from a single point into two points. As $T$ decreases further, the asymptotic oscillation period doubles from two to four, four to eight, etc, and eventually the process approaches an infinite period at another temperature threshold. \Fig{fig:bif} shows an example of the attractor bifurcation phenomenon. The example model has 4 discrete states and each state is associated with 2 real valued features which are randomly sampled from $\cN(0,1)$. Starting from that second threshold, the trajectory of $\bw$ is still bounded in a finite region as shown shortly in \Sec{sec:pct} but will not be periodic any more. Instead, we observe that the dynamics often converges to a fractal attractor set as shown in the right plot of \Fig{fig:bif}. The bifurcation process is observed very often in simulated models although it is not clear to us if it always happens for any discrete MRF. We discuss the dynamics related to this phenomenon in more detail in \Sec{sec:weak_chaos}.

\begin{figure}[tb!]
  \centering
    \includegraphics[width=0.48\textwidth]{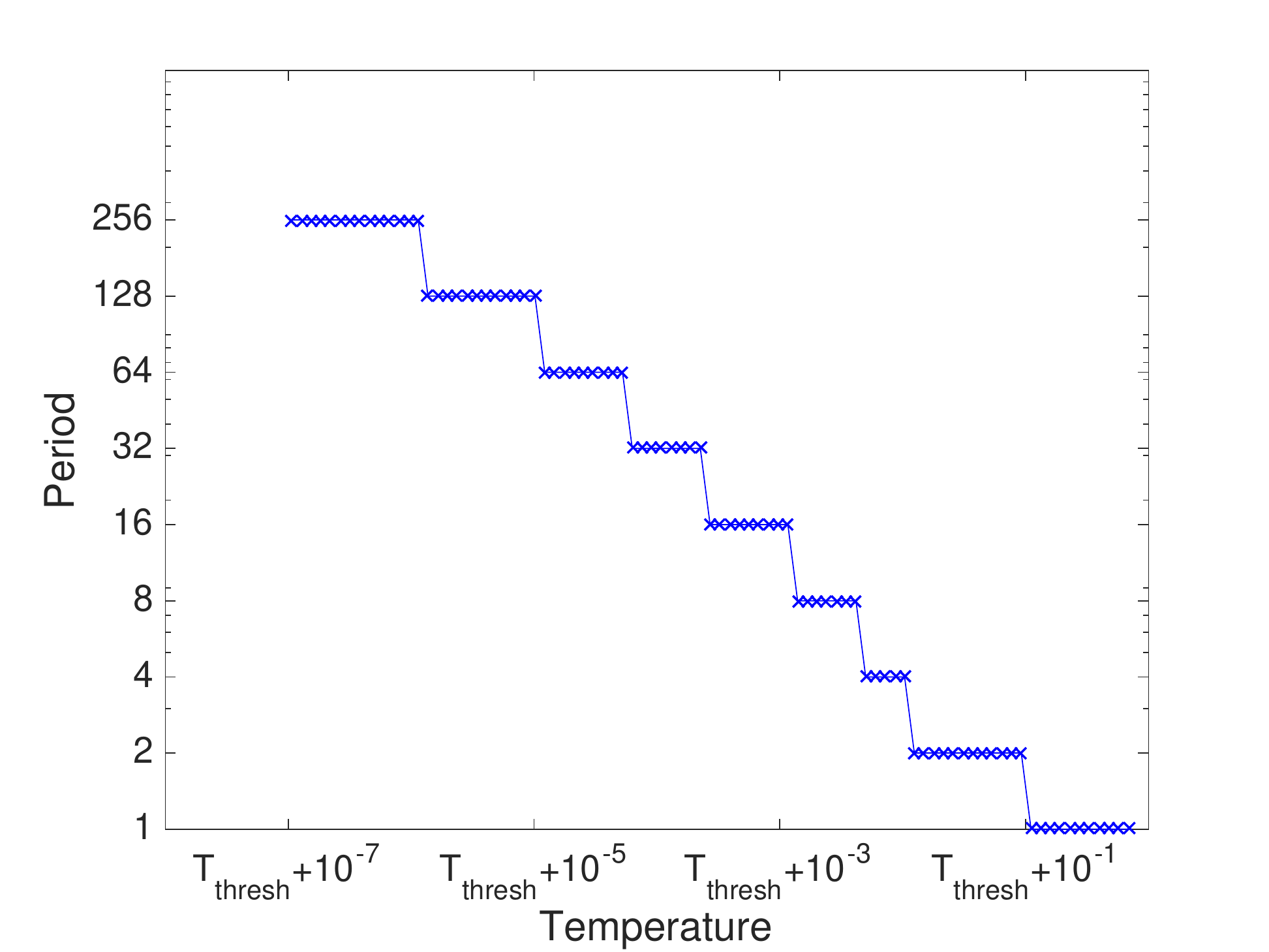}
  \hfill
    \includegraphics[width=0.48\textwidth]{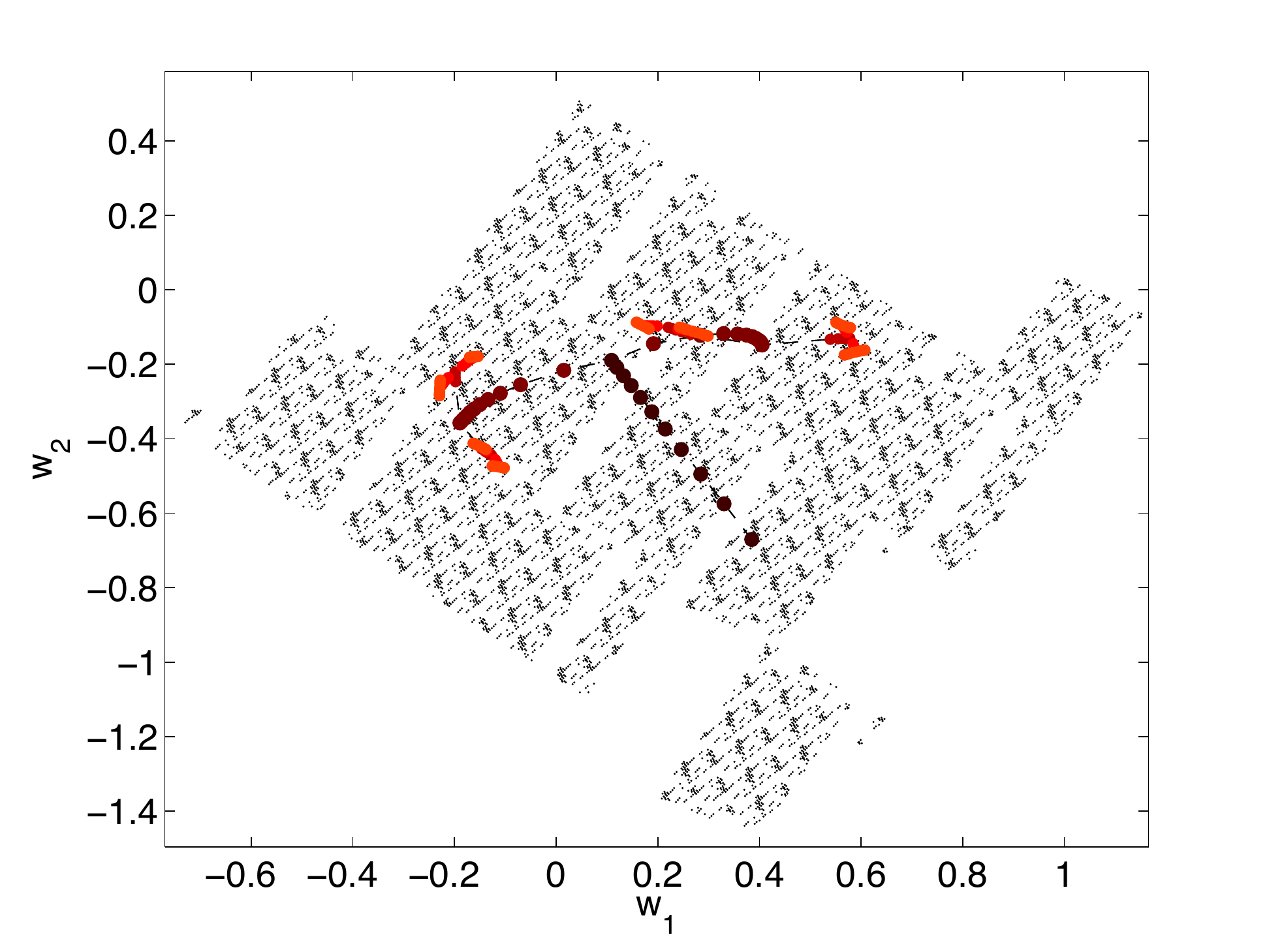} 
  \caption{Attractor bifurcation for a model with $4$ states and $2$-dimensional feature vectors. Left: Asymptotic period of the weight sequence (i.e. size of the attractor set) repeatedly doubles as the temperature decreases towards a threshold value (right to left). $T_{thresh}\approx 0.116$ in this example. The dynamics transits from periodic to aperiodic at that threshold. Right: The evolution of the attractor set of the weight sequence. As the temperature decreases (from dark to light colors), the attractor set split from a single point to two points, then to four, to eight, etc. The black dot cloud in the background is the attractor set at $T=0$.}
  \label{fig:bif}
\end{figure}

\subsection{Tipi Function and Basic Properties of Herding}\label{sec:basic_properties}
We will discuss a few distinguishing properties of the herding algorithm in this subsection. When we take the zero temperature limit in Equation \ref{eqn:llh_T}, the log-likelihood function becomes
\be
\ell_0(\bw) = \bw^T \bar\bphi - \max_{\bx}\left[\bw^T \bphi(\bx)\right] \label{eqn:llh_0}
\ee
This function has a number of interesting properties that justify the name ``Tipi function"\footnote{A Tipi is a traditional native Indian dwelling.} (see \Fig{fig:tipi}) \citep{Welling09A}:

\begin{figure}[tb!]
  \centering
  \includegraphics[width=0.5\textwidth]{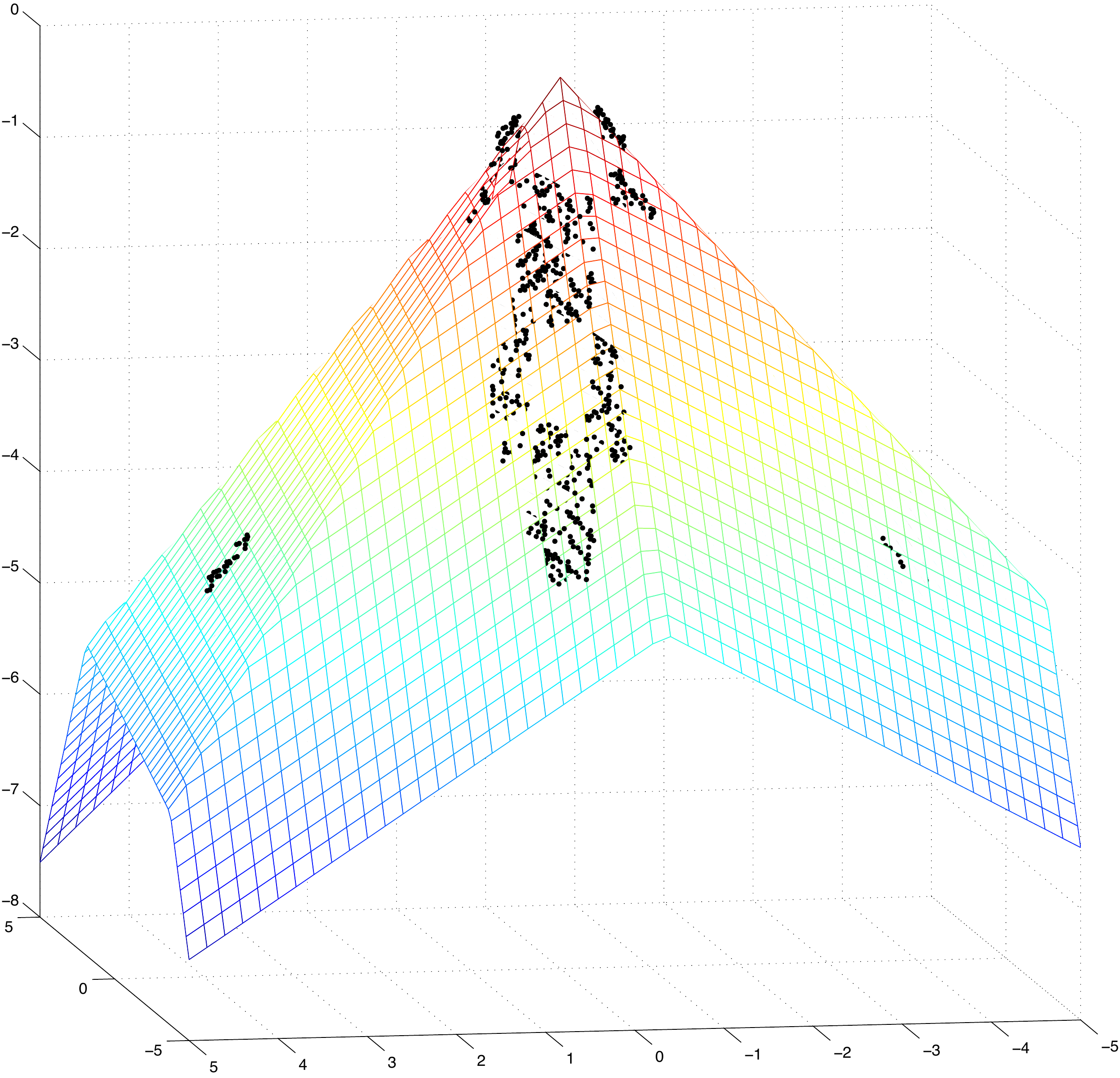}
  \caption{``Tipi function"\citep{Welling09A}: the log-likelihood function at the zero temperature limit. The black dots show the attractor set of the sequence of $\bw_t$.}
  \label{fig:tipi}
\end{figure}

\begin{enumerate}
\item \emph{$\ell_0$ is continuous piecewise linear ($C^0$ but not $C^1$).} It is clearly linear in $\bw$ as long as the maximizing state $\bs$ does not change. However, changing $\bw$ may in fact change the maximizing state in which case the gradient changes discontinuously.

\item \emph{$\ell_0$ is a concave, non-positive function of $\bw$ with a maximum at $\ell_0(\mathbf{0})=0$.} This is true because the first term represents the average $\mathbb{E}_P[\bw^T \bphi(\bx)]$ over some distribution P, while the second term is its maximum. Therefore, $\ell \leqq 0$. If we furthermore assume that $\bphi$ is not constant on the support of $P$ then $\ell_0<0$ and the maximum at $\bw=0$ is unique. Concavity follows because the first term is linear and the second maximization term is convex.
\item \emph{$\ell_0$ is scale free.} This follows because $\ell_0(\bt \bw)=\bt \ell_0(\bw),\forall \bt\geq 0$ as can be easily checked. This means that the function has exactly the same structure at any scale of $\bw$.
\end{enumerate}

Herding runs gradient descent optimization on this Tipi function. There is no need to search for the maximum as $\bw=0$ is the trivial solution. However, the fixed learning rate will always result in a perpetual overshooting of the maximum and thus the sequence of weights will never converge to a fixed point. Every flat face of the Tipi function is associated with a state. An important property of herding is that the state sequence visited by the gradient descent procedure satisfies the moment matching constraints in Equation \ref{eqn:maxent}, which will be discussed in details in Section \ref{sec:moment_matching}. There are a few more properties of this procedure that are worth noticing.

\subsubsection*{Deterministic Nonlinear Dynamics}
Herding is a deterministic nonlinear dynamical system. In contrast to the stochastic MLE learning algorithm based on MCMC, the two update steps in Equation \ref{eqn:herding_1} and \ref{eqn:herding_2} consist of a nonlinear deterministic mapping of the weights as illustrated in \Fig{fig:dynamic_system}. In particular it is not an MCMC procedure and it does not require random number generation.

The dynamics thus produces pseudo-samples that look random, but should not be interpreted as random samples. Although reminiscent of the Bayesian approach, the weights generated during this dynamics should not be interpreted as samples from some Bayesian posterior distribution. We will discuss the weakly chaotic behavior of the herding dynamics in detail in \Sec{sec:weak_chaos}.

\begin{figure}[htb!]
  \centering
  \includegraphics[width=0.5\textwidth]{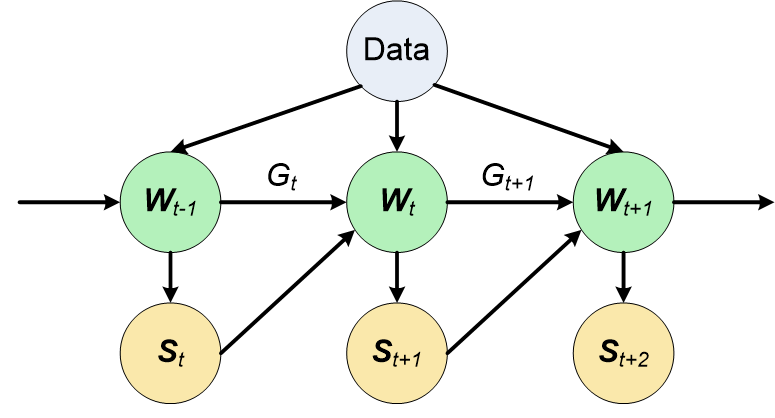}
  \caption{Herding as a nonlinear dynamical system.}
  \label{fig:dynamic_system}
\end{figure}

\subsubsection*{Invariance to the Learning Rate}
Varying the learning rate $\eta$ does not change the behavior of the herding dynamics. 
The only effect is to change the scale of the invariant attractor set of the sequence $\bw_t$. This actually follows naturally from the scale-free property of the Tipi function. More precisely, denote with $\bv_t$ the standard herding sequence with $\eta=1$ and $\bw_t$ the sequence with an arbitrary learning rate. It is easy to see that if we initialize $\bv_{t=0} = \frac{1}{\eta}\bw_{t=0}$ and apply the respective herding updates for $\bw_t$ and $\bv_t$ afterwards, the relation $\bv_t=\frac{1}{\eta} \bw_t$ will remain true for all $t>0$. In particular, the states $\bs_t$ will be the same for both sequences. Therefore we simply set $\eta=1$ in the herding algorithm.

Of course, if one initializes both sequences with arbitrary different values, then the state sequences will not be identical. However, if one accepts the conjecture that there is a unique invariant attractor set, then this difference can be interpreted as a difference in initialization which only affects the transient behavior (or ``burn-in" behavior) but not the (marginal) distribution $P(\bs)$ from which the states $\bs_t$ will be sampled.

Notice however that if we assign different learning rates $\{\eta_\alpha\}$ across the dimensions of the weight vector $\{w_\al\}$, it will change the distribution $P(\bs)$. While the moment matching constraints are still satisfied, we notice that the entropy of the sample distribution varies as a function of $\{\eta_\al\}$. In fact, changing the relative ratio of learning rates among feature dimensions is equivalent to scaling features with different factors in the greedy moment matching algorithm interpretation of \Sec{sec:greedy_alg}. How to choose an optimal set of learning rates is still an open problem.

\subsubsection*{Negative Auto-correlation}
A key advantage of the herding algorithm we observed in practice over sampling using a Markov chain is that the dynamical system mixes very rapidly over the attractor set. This is attributed to the fact that maximizations are performed on an ever changing model distribution as briefly mentioned at the beginning of this subsection. Let $\pi(\bx)$ be the distribution of training data $\mathcal{D}$, and $\bs_t$ be the maximizing state at time $t$. The distribution of an MRF at time $t$ with a regular temperature $T=1$ is
\be
P(\bx;\bw_{t-1})\propto \exp(\bw_{t-1}^T \bphi(\bx)) \label{eqn:old_px}
\ee
After the weights are updated with Equation \ref{eqn:herding_2}, the probability of the new model becomes
\begin{align}
& P(\bx;\bw_t) \propto \exp(\bw_{t}^T\bphi(\bx)) = \exp((\bw_{t-1}+\bar\bphi-\bphi(\bs_t))^T\bphi(\bx)) \nn\\
 &= \exp\left(\bw_{t-1}^T \bphi(\bx) +\sum_{\by\neq \bs_t}\pi(\by)\bphi(\by)^T\bphi(\bx) - (1 - \pi(\bs_t)) \bphi(\bs_t)^T\bphi(\bx)\right) \label{eqn:new_px}
\end{align}
Comparing Equation \ref{eqn:new_px} with \ref{eqn:old_px} we see that probable states (with large $\pi(\bx)$) are rewarded with an extra positive term $\pi(\bx)\bphi(\bx)^T\bphi(\bx)$, \emph{except} the most recently sampled state $\bs_t$. This will have the effect (after normalization) that state $\bs_t$ will have a smaller probability of being selected again. Imagine for instance that the sampler is stuck at a local mode. After drawing samples at that mode for a while, weights are updated to gradually reduce that mode and help the sampler escape it. The resulting negative auto-correlation would help mitigate the notorious problem of positive auto-correlation in most MCMC methods.

We illustrate the negative auto-correlation using a synthetic MRF with $10$ discrete states, each associated with a $7$-dimensional feature vector. The parameters of the MRF model are randomly generated from which the expected feature values are then computed analytically and fed into the herding algorithm to draw $T=10^5$ samples. We define the auto-correlation of the sample sequence of discrete variables as follows:
\begin{equation}
R(t) = \frac{\frac{1}{T-t}\sum_{\tau=1}^{T-t} \eI[s_\tau = s_{\tau + t}] - \sum_{s}\hat{P}(s)^2}{1 - \sum_{s}\hat{P}(s)^2}
\end{equation}
where $\eI$ is the indication function and $\hat{P}$ is the empirical distribution of the $10^5$ samples. It is easy to observe that $R(t=0)=1$ and if the samples are independently distributed $R(t)=0,\forall t>0$ up to a small error due to the finite sample size. We run herding $100$ times with different model parameters and show the mean and standard deviation of the auto-correlation in \Fig{fig:neg_autocorr}. We can see that the auto-correlation is negative for neighboring samples, and converges to 0 as the time lag increases. This effect exists even if we use a local optimization algorithm when a global optimum is hard or expensive to be obtained. This type of ``self-avoidance" is also shared with other sampling methods such as over-relaxation \citep{young1954iterative}, fast-weights PCD \citep{TielemanHinton09} and adaptive MCMC \citep{salakhutdinov27learning}.

\begin{figure}[htb!]
  \centering
  \includegraphics[width=0.5\textwidth]{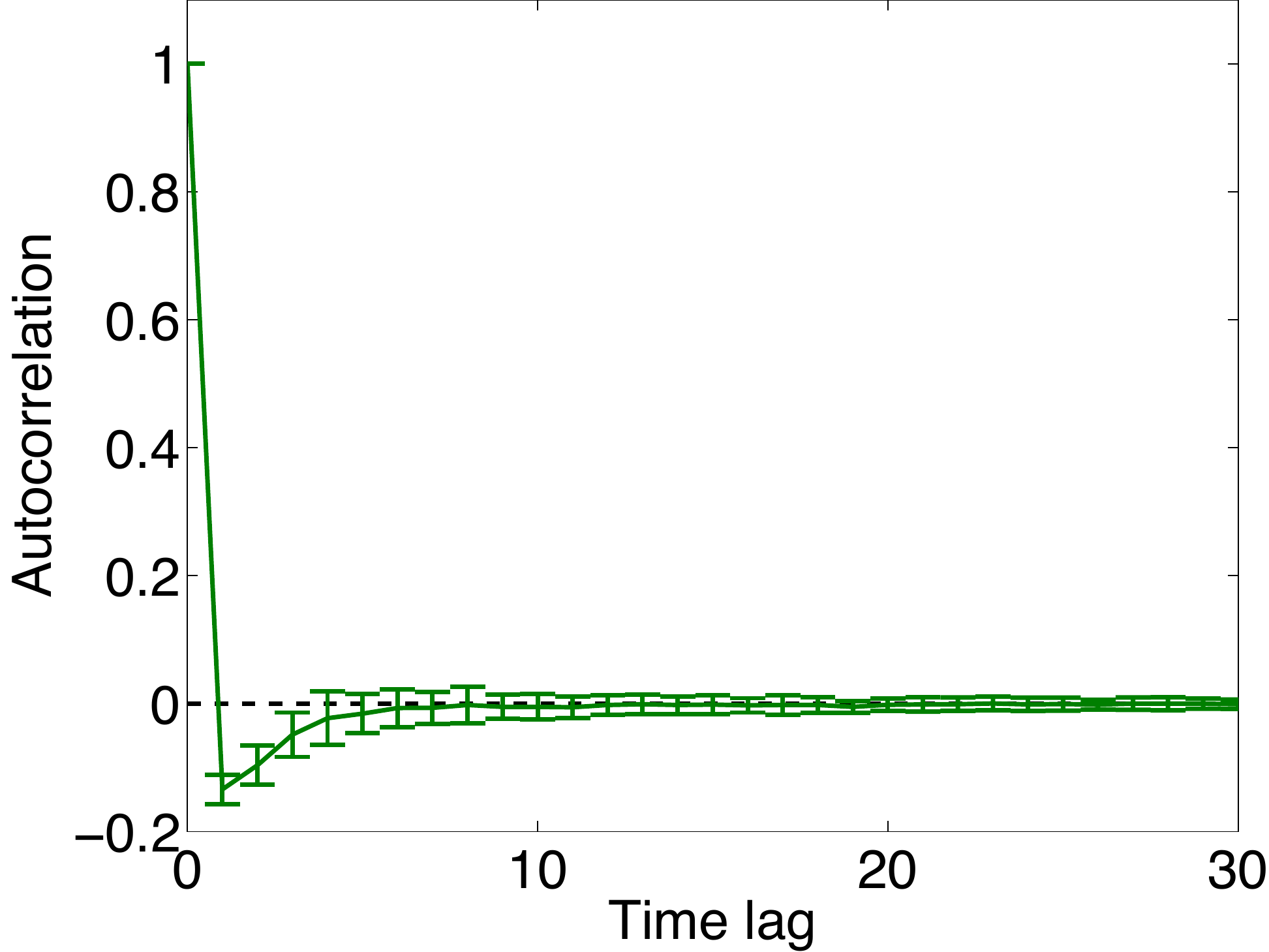}
  \caption{Negative auto-correlation of herding samples from a synthetic MRF.}
  \label{fig:neg_autocorr}
\end{figure}

\subsection{Herding as a Greedy Moment Matching Algorithm} \label{sec:greedy_alg}
As herding does not obtain the MLE, the distribution of the generated samples does not provide a solution to the maximum entropy problem either. However, we observe that the moment matching constraints in Equation \ref{eqn:maxent} are still respected, that is, when we compute the sampling average of the feature vector it will converge to the input moments. Furthermore, the negative auto-correlation in the sample sequence helps to achieve a convergence rate that is faster than what one would get from independently drawing samples or running MCMC at the MLE. Before providing any quantitative results, it would be easier for us to understand herding intuitively by taking a ``dual view" of its dynamics where we remove weights $\bw$ in favor of the states $\bx$ \citep{ChenSmolaWelling10}.

Notice that the expression of $\bw_T$ can be expanded recursively using the update Equation \ref{eqn:herding_2}:
\begin{equation}
\bw_{T} = \bw_0 + T \bar\bphi - \sum_{t=1}^T \bphi(\bs_t) \label{eqn:herding_2_expanded}
\end{equation}
Plugging \ref{eqn:herding_2_expanded} into Equation \ref{eqn:herding_1} results in
\begin{equation}
\bs_{T+1}=\arg\max_{\bx} \langle \bw_0, \bphi(\bx)\rangle + T\langle\bar\bphi, \bphi(\bx)\rangle - \sum_{t=1}^T \langle \bphi(\bs_t), \bphi(\bx)\rangle \label{eqn:herding_1_dual}
\end{equation}

For ease of intuitive understanding of herding, we temporarily make the assumptions (which are not necessary for the propositions to hold in the next subsection):
\begin{enumerate}
\item $\bw_0 = \bar\bphi$
\item $\|\bphi(\bx)\|_2=R, \forall \bx\in\mathcal{X}$
\end{enumerate}
The second assumption is easily achieved, e.g. by renormalizing $\bphi(\bx) \leftarrow \frac{\bphi(\bx)}{\|\bphi(\bx)\|}$ or by choosing a suitable feature map $\bphi$ in the first place. Given the first assumption, Equation \ref{eqn:herding_1_dual} becomes
\begin{equation}
\bs_{T+1}=\arg\max_{\bx} \langle\bar\bphi, \bphi(\bx)\rangle - \frac{1}{T+1} \sum_{t=1}^T \langle \bphi(\bs_t), \bphi(\bx)\rangle \label{eqn:herding_1_dual_2}
\end{equation}

Combining the second assumption one can show that the herding update equation \ref{eqn:herding_1_dual_2} is equivalent to greedily minimizing the squared error $\mathcal{E}_T^2$ defined as
\begin{equation}
\mathcal{E}_T^2 \defeq \left\|\bar\bphi - \frac{1}{T}\sum_{t=1}^T \bphi(\bs_t)\right\|^2 \label{eqn:error_l2}
\end{equation}

We therefore see that herding will generate pseudo-samples that greedily minimize the distance between the input moments and the sampling average of the feature vector at every iteration (conditioned on past samples). Note that the error function is unfortunately not submodular and the greedy procedure does not imply that the total collection of samples at iteration $T$ is jointly optimal (see \cite{Huszar12} for a detailed discussion). We also note that herding is an ``infinite memory process" on $\bs_t$ (as opposed to a Markov process) illustrated in \Fig{fig:dynamic_system_infinite_mem} because new samples depend on the entire history of samples generated thus far.

\begin{figure}[tb!]
  \centering
  \includegraphics[width=0.5\textwidth]{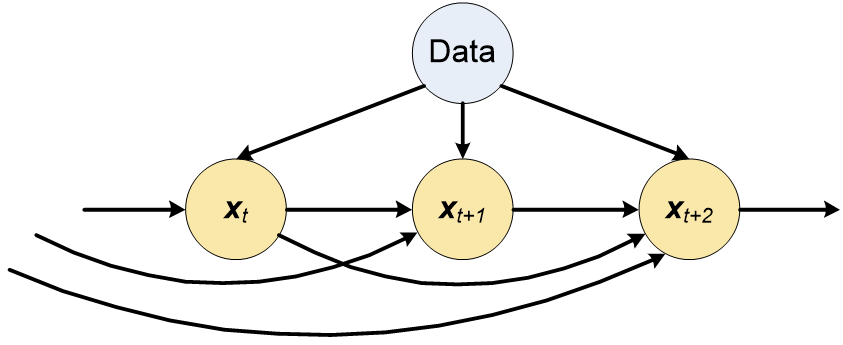}
  \caption{Herding as an infinite memory process on samples.}
  \label{fig:dynamic_system_infinite_mem}
\end{figure}

\subsection{Moment Matching Property} \label{sec:moment_matching}
With the dual view in the previous subsection, the distance between the moments and their sampling average in Equation \ref{eqn:error_l2} can be considered as the objective function for the herding algorithm to minimize. We discuss in this subsection under what condition and at what speed the moment constraints will be eventually satisfied.

\begin{proposition}[Proposition 1 in \cite{Welling09A}] \label{prop:sublinear_weight}
$\forall \al$, if~~$\lim_{\tau\rightarrow \infty}\dfrac{1}{\tau}w_{\al\tau}=0$, then $\lim_{\tau\rightarrow \infty} \dfrac{1}{\tau} \sum_{t=1}^\tau \phi_\al(\bs_t) = \bar\phi_\al$.
\end{proposition}

\begin{proof}
Following Equation \ref{eqn:herding_2_expanded}, we have
\be
\frac{1}{\tau}w_{\al \tau} - \frac{1}{\tau}w_{\al 0} = \bar\phi_\al - \frac{1}{\tau}\sum_{t=1}^\tau \phi_\al(\bs_t) \label{eqn:proof_converge}
\ee
Using the premise that the weights grow slower than linearly and observing that $w_{\al 0}$ is constant we see that the left hand term vanishes in the limit $\tau\rightarrow\infty$ which proves the result.
\end{proof}

What this says is that under the very general assumption that the weights do not grow linearly to infinity (note that due to the finite learning rate they can not grow faster than linear either), the moment constraints will be satisfied by the samples collected from the combined learning/sampling procedure. In fact, we will show later that the weights are restricted in a bounded region, which leads to a convergence rate of $\mathcal{O}(1/\tau)$ as stated below.
\begin{proposition} \label{prop:bounded_weight}
$\forall \al$, if there exists a constant $R$ such that $|w_{\al,t}|\leq R, \forall t$, then
$$\left|\dfrac{1}{\tau}\sum_{t=1}^\tau \phi_\al(\bs_t) - \bar\phi_\al\right| \leq \dfrac{2R}{\tau}.$$
\end{proposition}
The proof follows immediately Equation \ref{eqn:proof_converge}.

Note that if we want to estimate the expected feature of a trained MRF by a Monte Carlo method, the optimal standard deviation of the approximation error with independent and identically distributed (i.i.d.) random samples decays as $\mathcal{O}(\frac{1}{\sqrt{\tau}})$, where $\tau$ is the number of samples. (For positively autocorrelated MCMC methods this rate could be even slower.) Samples from herding therefore achieve a faster convergence rate in estimating moments than i.i.d. samples.

\subsubsection*{Recurrence of the Weight Sequence}
It is important to ensure that the herding dynamics does not diverge to infinity. \cite{Welling09A} discovered an important property of herding, known as recurrence, that the sequence of the weights is confined in a ball in the parameter space. This property satisfies the premise of both Proposition 2.1 and 2.2. It was stated in a corollary of Proposition \ref{prop:recurrence}:
\begin{proposition}[Proposition 2 in \cite{Welling09A}]\label{prop:recurrence}
$\exists \cR$ such that a herding update performed outside this radius will always decrease the norm $\|\bw\|_2$.
\end{proposition}
\begin{corollary}[Corollary in \cite{Welling09A}]\label{cor:recurrence}
$\exists \cR'$ such that a herding algorithm initialized inside a ball with that radius will never generate weights $\bw$ with norm $\|\bw\|_2 > \cR'$.
\end{corollary}
However, there was a gap in the proof of Proposition 2 in \cite{Welling09A}. We give the corrected proof below:
\begin{proof}[Proof of Proposition \ref{prop:recurrence}]
Write the herding update equation \ref{eqn:herding_2} as $\bw_{t}=\bw_{t-1} + \nabla_{\bw}\ell_0(\bw_{t-1})$ (set $\eta=1$). Expanding the squared norm of $\bw_t$ leads to
\begin{align}
&\|\bw_t\|_2^2 = \|\bw_{t-1}\|_2^2 + 2\bw_{t-1}^T \nabla_{\bw}\ell_0(\bw_{t-1}) + \|\nabla_{\bw}\ell_0(\bw_{t-1})\|_2^2\nn\\
&\Longrightarrow \quad \de\|\bw\|_2^2 < 2\ell_0(\bw_{t-1}) + \cB^2 \label{eqn:recurrence_dw}
\end{align}
where we define $\de\|\bw\|_2^2 = \|\bw_t\|_2^2 - \|\bw_{t-1}\|_2^2$. $\cB$ is an upper bound of $\{\|\nabla_{\bw}\ell_0(\bw)\|_2:\bw\in\cR^{|\bw|}\}$ introduced in Lemma 1 of \cite{Welling09A}. That exists as long as the norm of the feature vector $\bphi(\bx)$ is bounded in $\cX$. We also use the fact that $\ell_0(\bw) = \bw^T \nabla_{\bw}\ell_0(\bw)$.

Denote the unit hypersphere as $U=\{\bw|\|\bw\|_2^2=1\}$. Since $\ell_0$ is continuous on $U$ and $U$ is a bounded closed set, $\ell_0$ can achieve its supremum on $U$, that is, we can find a maximum point $w^*$ on $U$ where $\ell_0(\bw^*)\geq \ell_0(\bw), \forall \bw\in U$.

Combining this with the fact that $\ell_0 < 0$ outside the origin, we know the maximum of $\ell_0$ on $U$ is negative. Now taking into account the fact that $\cB$ is constant (i.e. does not scale with $\bw$), there exists some constant $\cR$ for which $\cR \ell_0(\bw^*) < -\cB^2/2$. Together with the scaling property of $\ell_0$, $\ell_0(\bt \bw) = \bt \ell_0(\bw)$, we can prove that for any $\bw$ with a norm larger than $\cR$,  $\ell_0$ is smaller then $-\cB^2/2$:
\begin{equation}
\ell_0(\bw)=\|\bw\|_2\ell_0(\bw/\|\bw\|_2)\leq \cR \ell_0(\bw^*) < -\cB^2/2,\quad \forall \|\bw\|_2>R
\end{equation}
The proof is concluded by plugging the inequality above in Equation \ref{eqn:recurrence_dw}.
\end{proof}

Corollary \ref{cor:recurrence} proves the existence of a bound for $\|\bw\|_2$ and thereby the constant $R$ in Proposition \ref{prop:bounded_weight}. \cite{harvey2014near} further studied the value of $R$ and proposed a variant of herding that obtained a near-optimal value for $R=O(\sqrt{d}\log^{2.5}\|\cX\|)$ w.r.t.\ the dimensionality of the feature vector $d$ and the size of a finite state space $\cX$. The proposed algorithm has a polynomial time complexity in $d$ and $\|\cX\|$.

\subsubsection*{The Remaining Degrees of Freedom}

Both the herding and the MaxEnt methods match the moments of the training data. But how does herding control the remaining degrees of freedom that are otherwise controlled by maximizing the entropy in the MaxEnt method? This is unfortunately still an open problem. Apart from some heuristics there is currently no principled way to enforce high entropy. In practice however, in discrete state spaces we usually observe that the sampling distribution from herding renders high entropy. We illustrate the behavior of herding in the example of simulating an Ising model in the next paragraph.

An Ising model is an MRF defined on a lattice of binary nodes, $G=(E, V)$, with biases and pairwise features. The probability distribution is expressed as
\be
P(\bx) = \frac{1}{Z}\exp\left(\bt\left(\sum_{(i,j)\in E}J_{i,j}x_i x_j + \sum_{i\in V} h_i x_i\right)\right), x_i \in \{-1, 1\}, \forall i\in V
\ee
where $h_i$ is the bias parameter, $J_{i,j}$ is the pairwise parameter and $\bt \ge 0$ is the inverse temperature variable. When $h_i=0$, $J_{i,j}=1$ for all nodes and edges, and $\bt$ is set at the inverse critical temperature, the Ising model is said to be at a critical phase where regular sampling algorithms fail due to long range correlations among variables. A special algorithm, the Swendsen-Wang algorithm \citep{swendsen1987nonuniversal}, was designed to draw samples efficiently in this case. In order to run herding on the Ising model, we need to know the average features, $\bar{x}_i$ ($0$ in this case) and $\overline{x_i x_j}$ instead of the MRF parameters. So we first run the Swendsen-Wang algorithm to obtain an estimate of the expected cross terms, $\overline{x_i x_j}$, which are constant across all edges, and then run herding with weights for every node $w_i$ and edge $w_{i,j}$. The update equations are:
\begin{align}
& \bs_t = \argmax_\bx \sum_{(i,j)\in E} w_{(i,j),t-1}x_i x_j + \sum_{i\in V} w_{i,t-1} x_i \\
& w_{(i,j),t} = w_{(i,j),t-1} + \overline{x_i x_j} - s_{i,t} s_{j,t} \\
& w_{i,t} = w_{i,t-1} - s_{i,t}
\end{align}
As finding the global optimum is an NP-hard problem we find a local maximum for $\bs_t$ by coordinate descent%
\footnote{In Section \ref{sec:generalize_herding} we show that the moment matching property still holds with a local search as long as the found state is better than the average.}%
. \Fig{fig:ising} shows a sample from an Ising model on an $100\times 100$ lattice at the critical temperature. We do not observe qualitative difference between the samples generated by the Ising model (MaxEnt) and herding, which suggests that the sample distribution of herding may be very close to the distribution of the MRF. Furthermore, \Fig{fig:ising_hist} shows the distribution of the size of connected components in the samples. It is known that this distribution should obey a power law at the critical temperature. We find that samples from both methods exhibit the power law distribution with an almost identical exponent.

\begin{figure}[tb!]
  \centering
  \subfloat[Generated by Swendsen-Wang] {
    \includegraphics[width=0.48\textwidth]{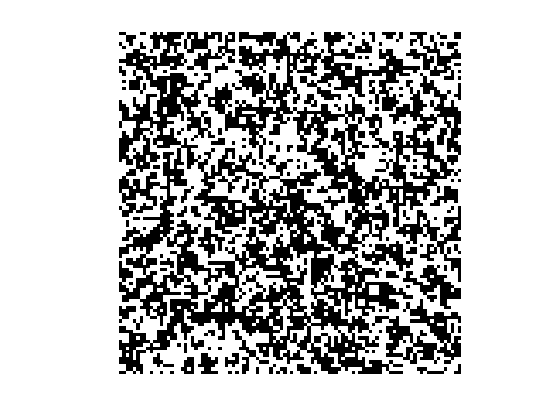}
  }
  \hfill
  \subfloat[Generated by Herding] {
    \includegraphics[width=0.48\textwidth]{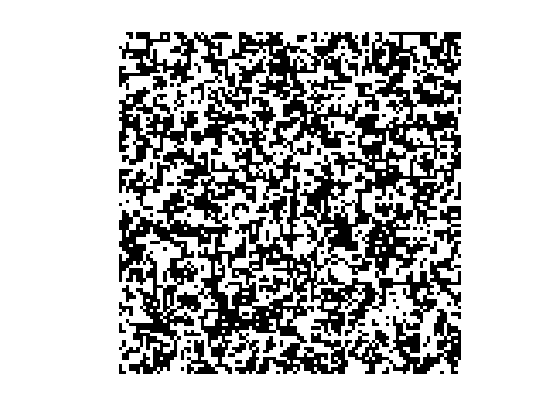}
  }
  \caption{Sample from an Ising model on an $100\times 100$ lattice at the critical temperature.}
  \label{fig:ising}
\end{figure}

\begin{figure}[tb!]
  \centering
  \subfloat[Generated by Swendsen-Wang] {
    \includegraphics[width=0.48\textwidth]{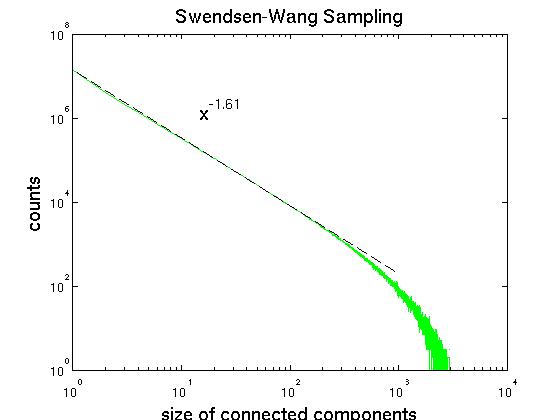}
  }
  \hfill
  \subfloat[Generated by Herding] {
    \includegraphics[width=0.48\textwidth]{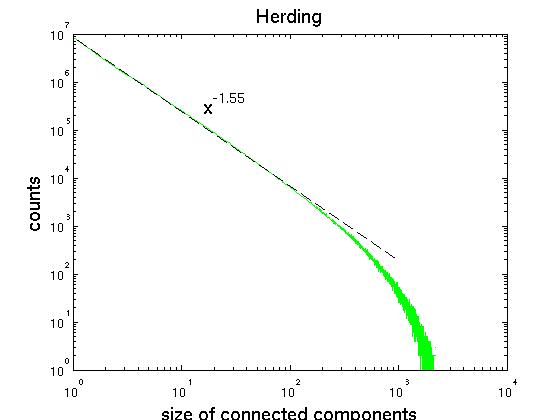}
  }
  \caption{Histogram of the size of connected components in the samples of the Ising model at the critical temperature.}
  \label{fig:ising_hist}
\end{figure}

\subsection{Learning Using Weak Chaos} \label{sec:weak_chaos}

There are two theoretical frameworks for statistical inference: the frequentist and the Bayesian paradigm. A frequentist assumes a true objective value for some parameter and tries to estimate its value from samples. Except for the simplest models, estimation usually involves an iterative procedure where the value of the parameter is estimated with increasing precision. In information theoretic terms, this means that more and more information from the data is accumulated in more decimal places of the estimate. With a finite data-set, this process should stop at some scale because there is not enough information in the data that can be transferred into the decimal places of the parameter. If we continue anyway, we will overfit to the dataset at hand. In a Bayesian setting we entertain a posterior distribution over parameters, the spread, or more technically speaking, entropy, of which determines the amount of information it encodes. In Bayesian estimation, the spread automatically adapts itself to the amount of available information in the data. In both cases, the learning process itself can be viewed as a dynamical system. For a frequentist this means a convergent series of parameter estimates indexed by the learning iteration $\bw_1, \bw_2, \dots$. For a Bayesian running a MCMC procedure this means a stochastic process converging to some equilibrium distribution. Herding introduces a third possibility by encoding all the information in a deterministic nonlinear dynamical system. We focus on studying the weakly chaotic behavior of the herding dynamics in this subsection. The sequence of weights never converges but traces out a quasi-periodic trajectory on an attractor set which is often found to be of fractal dimension. In the language of iterated maps, $\bw_{t+1} = F(\bw_t)$, a (frequentist) optimization of some objective results in an attractor set that is a single point, Bayesian posterior inference results in a (posterior) probability distribution while herding will result in a (possibly fractal) attractor set which seems harder to meaningfully interpret as a probability distribution. 

\subsubsection*{Example: Herding a Single Neuron}
We first study an example of the herding dynamics in its simplest form and show its equivalence to some well-studied theories in mathematics. Consider a single (artificial) neuron, which can take on two distinct states: either it fires ($x=1$) or it does not fire ($x=0$). Assume that we want to simulate the activity of a neuron with an irrational firing rate, $\pi\in[0,1]$, that is, the average firing frequency approaches $\lim_{T\to \infty}\frac{1}{T}\sum_{t=1}^Ts_t=\pi$. We can achieve that by applying the herding algorithm with a one-dimensional feature $\phi(x)=x$ and feeding the input moment with the desired rate $\bar{\phi}= \pi$. Applying the update equations \ref{eqn:herding_1}-\ref{eqn:herding_2} we get the following dynamics:
\begin{align}
& s_t = \eI(w_{t-1} > 0) \label{eqn:herding_neuron_1}\\
& w_t = w_{t-1} + \pi - s_t \label{eqn:herding_neuron_2}
\end{align}
where $\eI[\cdot]$ is the indicator function. With the moment matching property we can show immediately that the firing rate converges to the desired value $\pi$ for any initial value of $w$. 
The update equations are illustrated in \Fig{fig:univariate_1}. This dynamics is a simple type of interval translation mapping (ITM) problem in mathematics \citep{boshernitzan1995interval}. In a general ITM problem, the invariant set of the dynamics often has a fractal dimension. But for this simple case, the invariant set is the entire interval $(\pi-1,\pi]$ if $\pi$ is an irrational number and a finite set if it is rational. As a neuron model, one can think of $w_t$ as a ``synaptic strength." At each iteration the synaptic strength increases by an amount $\pi$. When the synaptic strength rises above $0$, the neuron fires. If it fires its synaptic strength is depressed by a factor $1$. The value of $w_0$ only has some effect on the transient behavior of the resulting sequence $s_1,s_2,\dots$.

\begin{figure}[tb]
 \centering
 \includegraphics[width=.46\textwidth] {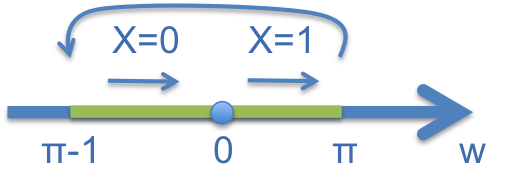}
 \caption{Herding dynamics for a single binary variable. At every iteration the weight is first increased by $\pi$. If $w$ was originally positive, it is then depressed by $1$.}
 \label{fig:univariate_1}
\end{figure}

It is perhaps interesting to note that by setting $\pi=\varphi$ with $\varphi$ the golden mean $\varphi=\ha(\sqrt{5}-1)$ and initializing the weights at $w_0=2\varphi-1$,  we exactly generate the ``Rabbit Sequence'': a well studied Sturmian sequence which is intimately related with Fibonacci numbers\footnote{Imagine two types of rabbits: young rabbits ($0$) and adult rabbits ($1$). At each new generation the young rabbits grow up ($0\ra 1$) and old rabbits produce offspring ($1\ra 10$). Recursively applying these rules we produce the rabbit sequence: $0\ra 1\ra 10 \ra 101 \ra 10110\ra 10110101$ etc. The total number of terms of these sequences and incidentally also the total number of $1$'s (lagged by one iteration) constitutes the Fibonacci sequence: $1,1,2,3,5,8,...$.}). In \Fig{fig:Fibonacci} we plot the weights (a) and the states (b) resulting from herding with the ``Fibonacci neuron" model. For a proof, please see \cite{WellingChen10}.
\begin{figure}[htb]
  \centering
    \includegraphics[width=0.48\textwidth]{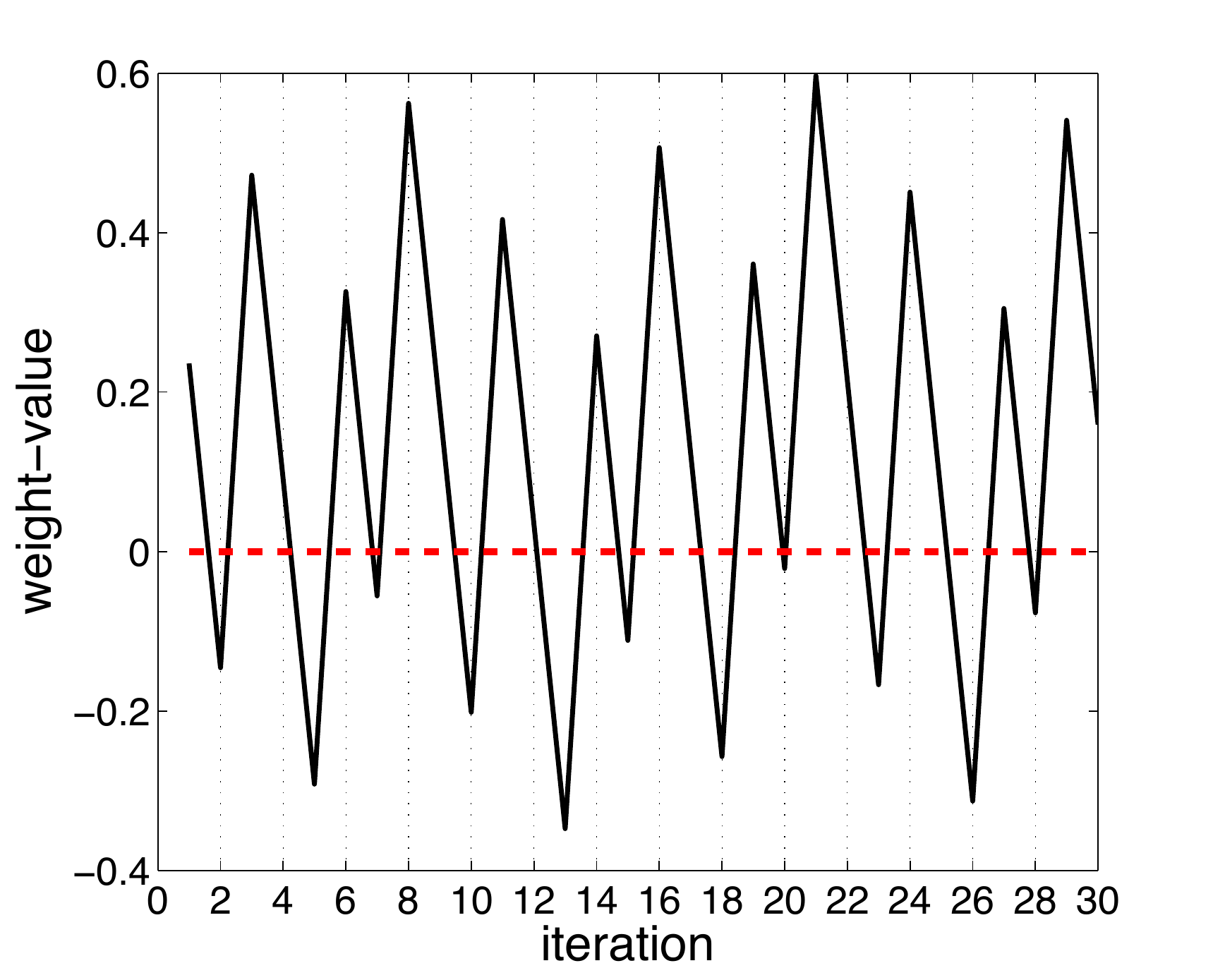}
  \hfill
    \includegraphics[width=0.48\textwidth]{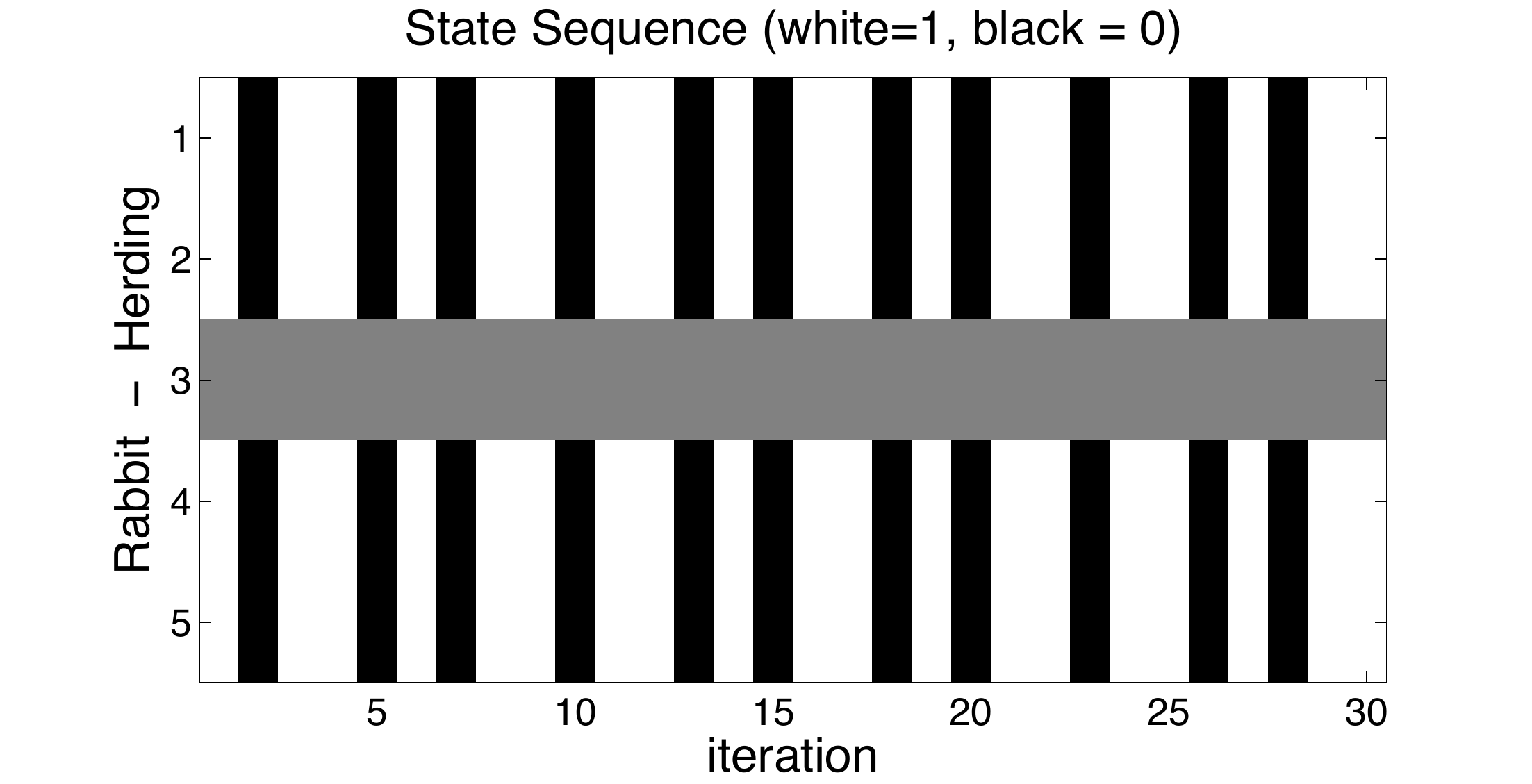}
  \caption{Sequence of weights and states generated by the ``Fibonacci neuron'' based on herding dynamics. Left: Sequence of weight values. Note that the state results by checking if the weight value is larger than $0$ (in which case $s_t=1$) or smaller than $0$ (in which case $s_t=0$). By initializing the weights at $w_0=2\varphi-1$ and using $\pi=\varphi$, with $\varphi$ the golden mean,  we obtain the Rabbit sequence (see main text). Right: Top stripes show the first $30$ iterates of the sequence obtained with herding. For comparison we also show the Rabbit sequence below it (white indicates $1$ and black indicates $0$). Note that these two sequences are identical.}
  \label{fig:Fibonacci}
\end{figure}

When initializing $w_0=0$, one may think of the synaptic strength as an error potential that keeps track of the total error so far. One can further show that the sequence of states is a discrete low discrepancy sequence \citep{angel2009discrete} in the following sense:
\begin{proposition} \label{prop:single_node} If $w$ is the weight of the herding dynamics for a single binary variable $x$ with probability $P(x=1)=\pi$, and $w_\tau \in (\pi-1, \pi]$ at some step $\tau\geq 0$, then $w_t\in(\pi-1,\pi],\forall t\geq \tau$. Moreover, for $T\in\mathbb{N}$, we have:
\be
\left|\sum_{t=\tau+1}^{\tau+T} \mathbb{I}[s_t=1] - T\pi \right| \leq 1,\quad 
\left|\sum_{t=\tau+1}^{\tau+T} \mathbb{I}[s_t=0] - T(1 - \pi) \right| \leq 1 \label{eqn:prop_single_node}
\ee
\end{proposition}
\begin{proof}
We first show that $(\pi-1, \pi]$ is the invariant interval for herding dynamics. Denote the mapping of the weight in Equation \ref{eqn:herding_neuron_1} and \ref{eqn:herding_neuron_2} as $\cT$. Then we can see that the interval $(\pi-1, \pi]$ is mapped to itself as
\be
\cT(\pi-1, \pi] = \cT(\pi-1, 0] \cup \cT(0, \pi] = (2\pi-1, \pi] \cup (\pi-1, 2\pi-1] = (\pi-1, \pi]
\ee
Consequently when $w_\tau$ falls inside the invariant interval, we have $w_t\in(\pi-1,\pi],\forall t\geq \tau$. Now summing up both sides of Equation \ref{eqn:herding_neuron_2} over $t$ immediately gives us the first inequality in \ref{eqn:prop_single_node} as:
\begin{equation}
T \pi - \sum_{t=\tau+1}^{\tau+T} \mathbb{I}[s_t=1] = w_{\tau+T}-w_\tau \in [-1, 1]. \label{eqn:prop_single_node_expand}
\end{equation}
The second inequality follows by observing that $\mathbb{I}[s_t=0] = 1 - \mathbb{I}[s_t=1]$.
\end{proof}

As a corollary of Proposition \ref{prop:single_node}, when we initialize $w_0=\pi - 1/2$, we can improve the bound of the discrepancy by a half.
\begin{corollary} \label{cor:single_node_half_bound} If $w$ is the weight of the herding dynamics in Proposition \ref{prop:single_node} and it is initialized at $w_0=\pi - 1/2$, then for $T\in\mathbb{N}$, we have:
\be
\left|\sum_{t=\tau+1}^{\tau+T} \mathbb{I}[s_t=1] - T\pi \right| \leq \ha,\quad 
\left|\sum_{t=\tau+1}^{\tau+T} \mathbb{I}[s_t=0] - T(1 - \pi) \right| \leq \ha \label{eqn:cor_single_node}
\ee
\end{corollary}
The proof immediately follows Equation \ref{eqn:prop_single_node_expand} by plugging $\tau=0$ and $w_0=\pi - 1/2$. In fact, setting $w_0=\pi - 1/2$ corresponds to the condition in the greedy algorithm interpretation in \Sec{sec:greedy_alg}. One can see this by constructing an equivalent herding dynamics with a feature of constant norm as:
\begin{equation}
\phi'(x) = \left\{\begin{array}{rl}
1 & \textrm{if~} x = 1\\
-1 & \textrm{if~} x = 0
\end{array} \right.
\end{equation}
When initializing the weight at the moment $w'_0 = \bar{\phi'} = 2\pi - 1$, one can verify that this dynamics generates the same sample sequence as the original one and their weights are the same up to a constant factor of $2$, i.e. $w'_t = 2 w_t, \forall t\geq 0$. The new dynamics satisfies the two assumptions in \Sec{sec:greedy_alg} and therefore the sample sequences in both dynamical systems greedily minimize the error of the empirical probability (up to a constant factor):
\begin{equation}
\left|\frac{1}{T}\sum_{t=1}^T \phi'(x'_t) - (2\pi-1) \right| = 2\left|\frac{1}{T}\sum_{t=1}^T \eI[x_t=1] - \pi \right|
\end{equation}
This greedy algorithm actually achieves the optimal bound one can get with herding dynamics in the 1-neuron model, which is $1/2$.

\subsubsection*{Example: Herding a Discrete State Variable}
The application of herding to a binary variable can be extended naturally to a discrete state variables. Let $x$ be a variable that can take one of the $D$ states, $\{0, 1, \dots, D-1\}$. Given any distribution over these $D$ states in the set $\bpi \in \mathbb{R}^D, \sum_{d=0}^{D-1} \pi_d = 1$, we can run herding to simulate the activity of the discrete variable. The feature function, $\bphi(x)$, is defined as the 1-of-D encoding of the discrete state, that is, a vector of $D$ binary numbers, in which all the numbers are 0 except for the element indexed by the value of $x$. For example, for a variable with 4 states, the feature function of $\bphi(x=3)$ is $[0, 0, 1, 0]$. It is easy to observe that the expected value of the feature vector under the distribution $\bpi$ is exactly equal to $\bpi$.  Now, let us apply the herding update equations with the feature map $\bphi$ and input moment $\bpi$:
\begin{align}
s_t &= \arg\max_x \bw_{t-1}^T \bphi(x) = \arg\max_x w_{x,t-1} \label{eqn:herding_multinomial_1}\\
\bw_t &= \bw_{t-1} + \bpi - \bphi(s_t) \label{eqn:herding_multinomial_2}
\end{align}
The weight variables act similarly to the synaptic strength analogy in the neuron model example. At every iteration, the state with the highest potential gets activated, and then the corresponding weight is depressed after activation. Applying Proposition \ref{prop:bounded_weight}, we know that the empirical distribution of the samples converges to the input distribution at a faster rate than one would get from random sampling:
\begin{equation}
\left|\frac{1}{T}\sum_{t=1}^T \bphi(s_t) - \bpi \right| = \mathcal{O}\left(\frac{1}{T}\right)
\end{equation}

The dynamics of the weight vector is more complex than the case of a binary variable in the previous subsection. However, there are still some interesting observations one can make about the trajectory of the weights which we explain in the appendix.

\subsubsection*{Weak Chaos in the Herding Dynamics}
Now let us consider herding in a general setting with $D$ states and each state is associated with a $K$ dimensional feature vector. The update equation for the weights \ref{eqn:herding_2} can be viewed as a series of translations in the parameter space, $\bw\rightarrow \bw + \rho(\bx)$, where each discrete state $\bx\in\cX$ corresponds to one translation vector (i.e. $\rho(\bx) = \bar\bphi - \bphi(\bx)$). See \Fig{fig:cones} for an example with $D=6$ and $K=2$. The parameter space is partitioned into cones emanating from the origin, each corresponding to a state according to Equation \ref{eqn:herding_1}. If the current location of the weights is inside cone $\bx$, then one applies the translation corresponding to that cone and moves along $\rho(\bx)$ to the next point. This system is an example of what is known as a piecewise translation (or piecewise isometry more generally) \citep{goetz2000dynamics}.

\begin{figure}[tb]
  \centering
  \begin{minipage}[t]{0.48\textwidth}
    \centering
    \includegraphics[width=\textwidth]{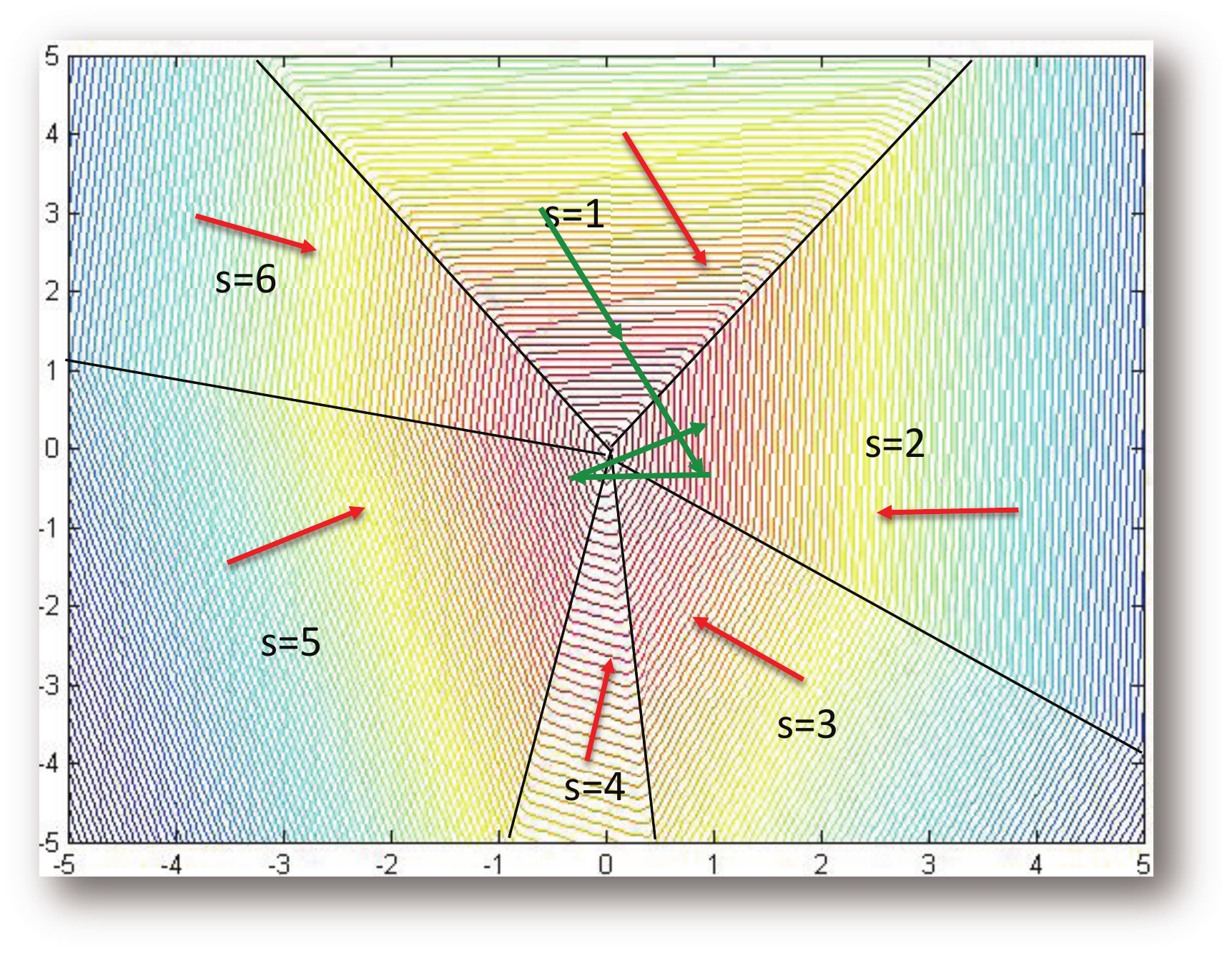}
    \caption{Cones in parameter space $\{w_1, w_2\}$ that correspond to the discrete states $s_1,...,s_6$. Arrows indicate the translation vectors associated with the cones.}
    \label{fig:cones}
  \end{minipage}
  \hfill
  \begin{minipage}[t]{0.48\textwidth}
    \centering
    \includegraphics[width=\textwidth]{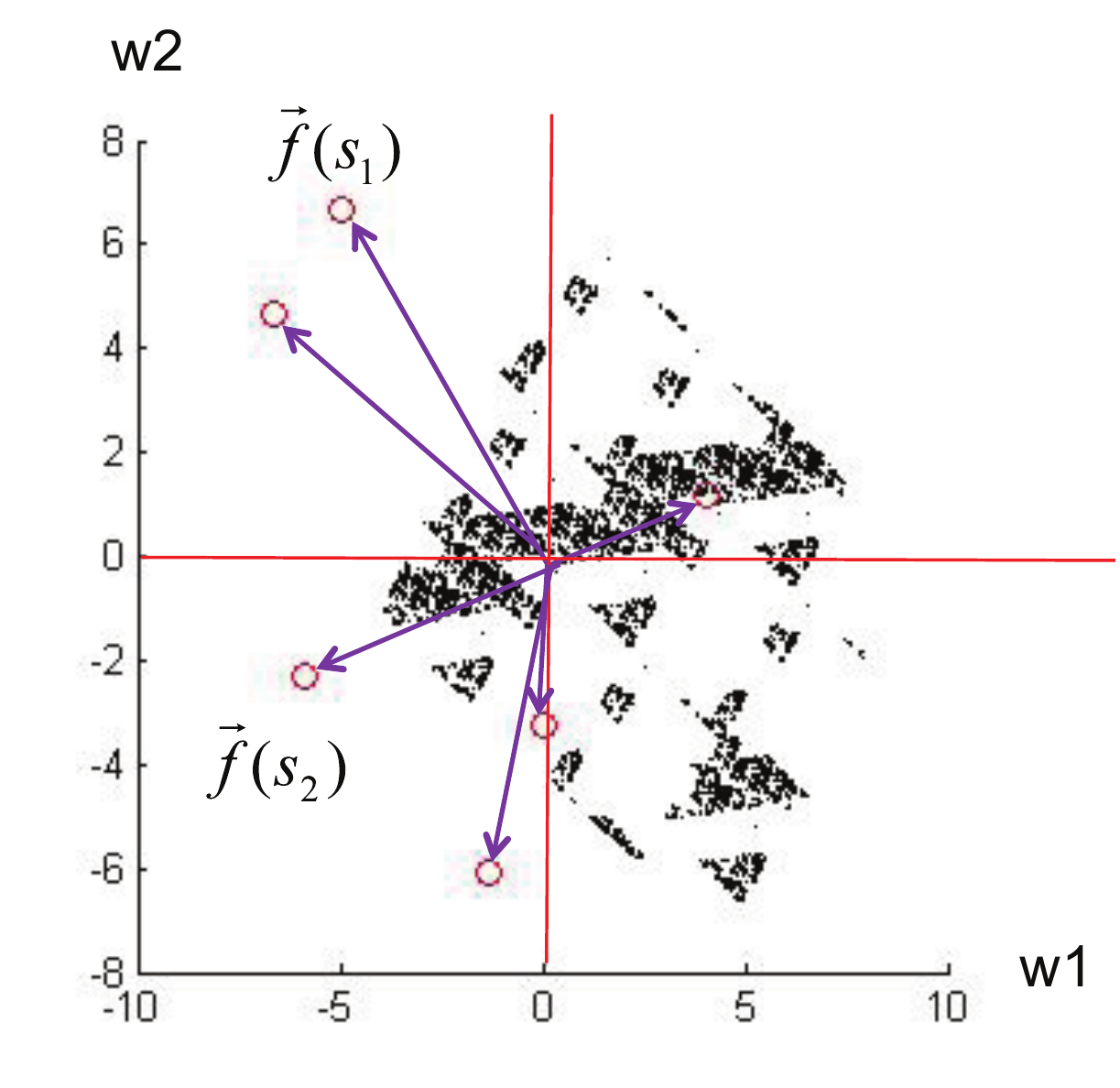}
    \caption{Fractal attractor set for herding with two parameters. The circles represent the feature-vectors evaluated at the states $s_1,...,s_6$. Hausdorff dimension for this example is between $0$ and $1$.}
    \label{fig:attractor}
  \end{minipage}  
\end{figure}


It is clear that this system has zero Lyapunov exponents\footnote{The Lyapunov exponent of a dynamical system is a quantity that characterizes the rate of separation of infinitesimally close trajectories. Quantitatively, two trajectories in phase space with initial separation $|\delta Z(0)|$ diverge (provided that the divergence can be treated within the linearized approximation) at a rate given by $|\delta Z(t)|\approx e^{\lambda t}|\delta Z(0)|$ where $\lambda$ is the Lyapunov exponent.} everywhere (except perhaps on the boundaries between cones but since this is a measure zero set we will ignore these). As the evolution of the weights will remain bounded inside some finite ball the evolution will converge to some attractor set. Moreover, the dynamics is non-periodic in the typical case (more formally, the translation vectors must form an incommensurate (possibly over-complete) basis set; for a proof see Appendix B of \cite{WellingChen10}). It can often be observed that this attractor has fractal dimension (see \Fig{fig:attractor} for an example). All these facts point to the idea that herding is on the edge between full chaos (with positive Lyapunov exponents) and regular periodic behavior (with negative Lyapunov exponents). In fact, herding is an example of what is called ``weak chaos", which is usually defined through its (topological) entropy discussed below. Finally, as we have illustrated in \Fig{fig:bif}, one can construct a sequence of iterated maps of which herding is the limit and which exhibits period doubling. This is yet another characteristic of systems that are classified as ``edge of chaos". Whether the attractor set is of fractal dimension in general remains an open question. For the case of single neuron model, the attractor is the entire interval $(\pi-1,\pi]$ if $\pi$ is irrational but for systems with more states it remains unknown.

We will now estimate the entropy production rate of herding. This will inform us further of the properties of this system and how it processes information. From \Fig{fig:cones} we see that the sequence $s_1,s_2,...$ can be interpreted as the symbolic system of the continuous dynamical system defined for the parameters $\bw$. A sequence of symbols (states) is sometimes referred to as an ``itinerary.'' Every time $\bw$ falls inside a cone we record its label which equals the state $\bx$. The topological entropy for the symbolic system can be defined by counting the total number of subsequences of length $T$, which we will call $M(T)$. One may think of this as a dynamical language where the subsequences are called ``words" and the topological entropy is thus related to the number of words of length $T$. More precisely, the topological entropy is defined as,
\be
h = \lim_{T\ra\infty} h(T) = \lim_{T\ra\infty} \frac{\log M(T)}{T}
\ee
It was rigorously proven in \cite{goetz2000dynamics} that $M(T)$ grows polynomially in $T$ for general piecewise isometries, which implies that the topological entropy vanishes for herding. It is however interesting to study the growth of $M(T)$ as a function of $T$ to get a sense of how chaotic its dynamics is.

For the simplest model of a single neruon with $\pi$ being an irrational number, it turns out $M(T) = T+1$, which is the absolute bare minimum for sequences that are not eventually periodic. It implies that our neuron model generates Sturmian sequences for irrational values of $\pi$ which are precisely defined to be the non-eventually periodic sequences of minimal complexity \citep{LuWang05}. (For a proof, please see \cite{WellingChen10}.)

To count the number of subsequences of length $T$ for a general model, we can study the $T$-step herding map that results from applying herding $T$ steps at a time. The original cones are now further subdivided into smaller convex polygons, each one labeled with the sequence $s_1,s_2,...,s_T$ that the points inside the polygon will follow during the following $T$ steps. Thus as we increase $T$, the number of these polygons will increase and it is exactly the number of those polygons which partition our parameter space that is equal to the number of possible subsequences. We first claim that every polygon, however small, will break up into smaller sub-pieces after a finite amount of time. This is proven in \cite{WellingChen10}. In fact, we expect that in a typical herding system every pair of points will break up as well, which, if true, would infer that the diameter of the polygons must shrink. A partition with this property is called a \emph{generating partition}. Based on some preliminary analysis and numerical simulations, we expect that the growth of $M(T)$ in the typical case (a.k.a. with an incommensurate translation basis, see Appendix B of \cite{WellingChen10}) is a polynomial function of the time, $M(T) \sim t^K$, where $K$ is the number of dimensions (which is equal to the number of herding parameters). Since it has been rigorously proven that any piecewise isometry has a growth rate that must have an exponent less or equal than $K$ \citep{goetz2000dynamics}, this would mean that herding achieves the highest possible entropy within this class of systems with $H(T)=T h(T)$ for a sequence of length $T$ (for $T$ large enough) as:
\be
H(T) = K\log(T) \label{eqn:entropy_herding}
\ee

This result should be understood in comparison with regular and random sequences. In a regular (constant or periodic) sequence, the number of subsequences is constant with respect to the length, i.e. $H(T) = \mathrm{const}$. In contrast, the dominant part of the Kolmogorov-Sinai entropy of a random sequence (considering, e.g., a stochastic process) or a fully chaotic sequence grows linearly in time $T$, i.e. $H_{\mathrm{ext}}(T)=hT$ due to the injected random noise.

\section{Generalized Herding}\label{sec:extensions}

The moment matching property in Proposition \ref{prop:sublinear_weight} and \ref{prop:bounded_weight} requires only a mild condition on the $L_2$ norm of the dynamic weights. That grants us with great flexibility in modifying the original algorithm for more practical implementation as well as a larger spectrum of applications. \cite{GelfandMaatenChenWelling10} provided a general condition on the recurrence of the weight sequence, from which we discuss how to generalize the herding algorithm in this section with two specific examples. \cite{chen2014herdingbookchapter} described another extension of herding that violated the condition but it achieved the minimum matching distance instead in a constrained problem.

\subsection{A General Condition for Recurrence - The Perceptron Cycling Theorem} \label{sec:pct}
The moment matching property of herding relies on the recurrence of the weight sequence (Corollary \ref{cor:recurrence}) whose proof again relies on the premise that the maximization is carried out exactly in the herding update equation \ref{eqn:herding_1}. However, the number of model states is usually exponentially large (e.g. $|\cX| = J^m$ when $\bx$ is a vector of $m$ discrete variables each with $J$ values) and it is intractable to find a global maximum in practice. A local maximizer has to be employed instead. One wonders if the features averaged over samples will still converge to the input moments when the samples are suboptimal states? In this subsection we give a general and verifiable condition for the recurrence of the weight sequence based on the perceptron cycling theorem \citep{minsky1969perceptrons}, which consequently suggests that the moment matching property may still hold at the rate of $\cO(1/T)$ even with a relaxed herding algorithm.


The invention of the perceptron \citep{rosenblatt1958perceptron} goes back to the very beginning of AI more than half a century ago. Rosenblatt's very simple, neurally plausible learning rule made it an attractive algorithm for learning relations in data: for every input $\bx_i$, make a linear prediction about its label: $y_{i_t}^* = \mathrm{sign}(\bw_{t-1}^T \bx_{i_t})$ and update the weights as,
\be
\bw_t = \bw_{t-1} + \bx_{i_t} (y_{i_t} - y_{i_t}^*). \label{eqn:perceptron}
\ee
A critical evaluation by \cite{minsky1969perceptrons} revealed the perceptron's limited representational power. This fact is reflected in the behavior of Rosenblatt's learning rule: if the data is linearly separable, then the learning rule converges to the correct solution in a number of iterations that can be bounded by $(R/\ga)^2$, where $R$ represents the norm of the largest input vector and $\ga$ represents the margin between the decision boundary and the closest data-case. However, ``for data sets that are not linearly separable, the perceptron learning algorithm will never converge" (quoted from \cite{bishop2006pattern}).

While the above result is true, the theorem in question has something much more powerful to say. The ``perceptron cycling theorem" (PCT) \citep{minsky1969perceptrons} states that for the inseparable case the weights remain bounded and do not diverge to infinity. The PCT was initially introduced in \cite{minsky1969perceptrons} but had a gap in the proof that was fixed in \cite{block1970boundedness}. 

\begin{theorem}[Boundedness Theorem]
Consider a sequence of vectors $\{\bw_t\}$, $\bw_t \in \mathbb{R}^D$, $t = 0, 1, \dots$ generated by the iterative procedure of \Algo{alg:pct_sequence}.
\begin{algorithm}
\caption{Algorithm to generate the sequence $\{\bw_t\}$.}
\label{alg:pct_sequence}
\begin{algorithmic}
\State $V$ is a finite set of vectors in $\mathbb{R}^D$.
\State $\bw_0$ is initialized arbitrarily in $\mathbb{R}^D$.
\For{$t = 0 \to T$ ($T$ could be $\infty$)}
\State $\bw_{t+1} = \bw_t + \bv_t$, where $\bv_t\in V$ satisfies $\bw_t^T \bv_t \leq 0$
\EndFor
\end{algorithmic}
\end{algorithm}

Then, $\|\bw_t\| \leq \|\bw_0\| + M, \forall t\geq 0$ where $M$ is a constant depending on $V$ but not on $\bw_0$.
\end{theorem}

The theorem still holds when $V$ is a finite set in a Hilbert space. The PCT leads to the boundedness of the perceptron weights where we identify $\bv_t = \bx_{i_{t+1}}(y_{i_{t+1}}-y_{i_{t+1}}^*)$, a finite set $V = \{\bx_i(y_i-y_i^*)|y_i=\pm 1,y_i^*=\pm 1, i=1,\dots,N\}$ and observe
\be
\bw_t^T \bv_t = \bw_t^T \bx_{i_{t+1}}(y_{i_{t+1}} - y_{i_{t+1}}^*) = |\bw_t^T \bx_{i_{t+1}}|(\mathrm{sign}(\bw_t^T \bx_{i_{t+1}})y_{i_{t+1}} - 1) \leq 0
\ee
When the data is linearly separable, Rosenblatt's learning rule will find a $\bw$ such that $\bw^T \bv_i=0, \forall i$ and the sequence of $\bw_t$ converges. Otherwise, there always exists some $\bv_i$ such that $\bw^T \bv_i<0$ and PCT guarantees the weights are bounded.

The same theorem also applies to the herding algorithm by identifying $\bv_t = \bar\bphi-\bphi(\bs_{t+1})$ with $\bs_{t+1}$ defined in Equation \ref{eqn:herding_1}, a finite set $V=\{\bar\bphi-\bphi(\bx)|\bx\in\mathcal{X}\}$, and observing that  
\be
\bw_t^T \bv_t = \bw_t^T \bar\bphi- \bw_t^T \bphi(\bs_{t+1}) \leq 0 \label{eqn:pct_inequality}
\ee
It is now easy to see that, in general, herding does not converge because under very mild conditions we can always find an $\bs_{t+1}$ such that $\bw_t^T \bv_t < 0$. More importantly, the boundedness theorem (or PCT) provides a general condition for the recurrence property and hence the moment matching property of herding. Inequality \ref{eqn:pct_inequality} is easy to be verified at running time and does not require $\bs_{t+1}$ to be the global optimum.

\subsection{Generalizing the Herding Algorithm} \label{sec:generalize_herding}

PCT ensures that the average features from the samples will match the moments at a fast convergence rate as long as the algorithm we are running satisfies the following conditions:
\begin{enumerate}
\item The set $V$ is finite,
\item $\bw_t^T \bv_t = \bw_t^T \bar\bphi - \bw_t^T \bphi(\bs_t) \leq 0, \forall t$, 
\end{enumerate}
This set of mild conditions allows us to generalize the original herding algorithm easily.

Firstly, the PCT provides a theoretical justification for using a local search algorithm that performs partial maximization. For example, we may start the local search from the state we ended up in during the previous iteration (a so-called persistent chain \citep{Younes89,Neal92,Yuille04,Tieleman08}). Or, one may consider contrastive divergence-like algorithms \citep{Hinton02}, in which the sampling or mean field approximation is replaced by a maximization. In this case, maximizations are initialized on all data-cases and the weights are updated by the difference between the average over the data-cases minus the average over the $\{\bs_i\}$ found after (partial) maximization. In this case, the set V is given by: $V = \{\bar\bphi - \frac{1}{D} \sum_{i=1}^D \bphi(\bs_i)| \bs_i \in \mathcal{X}, \forall i\}$. For obvious reasons, it is now guaranteed that $\bw_t^T \bv_t \leq 0$.

Secondly, we often use mini-batches of size $d < D$ in practice instead of the full data set. In this case, the cardinality of the set $V$ is enlarged to, e.g., $|V| = C(d,D)J^m$, with $C(d,D)$ representing the ``$d$ choose $D$" ways to compute the sample mean $\bar\bphi_{(d)}$ based on a subset of $d$ data-cases. The negative term remains unaltered. Since the PCT still applies: $\left\| \frac{1}{\tau} \sum_{t=1}^\tau \bar\bphi_{(d),t} - \frac{1}{\tau}\sum_{t=1}^\tau \bphi(\bs_t)\right\|_2 = \mathcal{O}(1/\tau)$. Depending on how the mini-batches are picked, convergence onto the overall mean $\bar\bphi$ can be either $
\mathcal{O}(1/\sqrt{\tau})$ (random sampling with replacement) or $\mathcal{O}(1/\tau)$ (sampling without replacement which has picked all data-cases after $\lceil D/d \rceil$ rounds).

Besides changing the way we compute the positive and negative terms in $\bv_t$, generalizing the definition of {\em features} will allow us to learn a much wider scope of models beyond the fully visible MRFs as discussed in the following sections.

\subsection{Herding Partially Observed Random Field Models}\label{sec:herding_pomrf}
The original herding algorithm only works for fully visible MRFs because in order to compute the average feature vector of the training data we have to observe the state of all the variables in a model. In this subsection, we generalize herding to partially observed MRFs (POMRFs) by dynamically imputing the value of latent variables in the training data during the run of herding. This extension allows herding to be applied to models with a higher representative capacity.

Consider a MRF with discrete random variables $(\bx, \bz)$ where $\bx$ will be observed and $\bz$ will remain hidden. A set of feature functions is defined on $\bx$ and $\bz$, $\{\phi_\al(\bx, \bz)\}$, each associated with a weight $w_\al$. Given these quantities we can write the following Gibbs distribution,
\be
P(\bx, \bz; \bw) = \frac{1}{Z(\bw)}\exp\left(\sum_\al w_\al \phi_\al(\bx, \bz)\right)
\ee
The log-likelihood function with a dataset $\mathcal{D}=\{\bx_i\}_{i=1}^D$ is defined as
\be
\ell(\bw) = \frac{1}{D}\sum_{i=1}^D \log\left(\sum_{\bz_i}\exp\left(\bw^T \bphi(\bx_i,\bz_i)\right)\right) - \log Z(\bw) \label{eqn:llh_pomrf}
\ee

Analogous to the duality relationship between MLE and MaxEnt for fully observed MRFs, we can write the log-likelihood of a POMRF as
\begin{align}
&\ell =\max_{\{Q_i\}}\min_{R} \frac{1}{D}\sum_{i=1}^D \mathcal{H}(Q_i) - \mathcal{H}(R) \\
&+ \sum_{\al} w_\al \left(\frac{1}{D}\sum_{i=1}^D \mathbb{E}_{Q_i(\bz_i)}[\phi_\al(\bx_i, \bz_i)] - \mathbb{E}_{R(\bx,\bz)}[\phi_\al(\bx, \bz)]\right) \label{eqn:maxent_pomrf}
\end{align}
where $\{Q_i\}$ are variational distributions on $\bz$, and $R$ is a variational distribution on $(\bx, \bz)$. The dual form of MLE turns out as a minimax problem on $\frac{1}{D}\sum_{i=1}^D \mathcal{H}(Q_i) - \mathcal{H}(R)$ with a set of constraints
\begin{equation}
\frac{1}{D}\sum_{i=1}^D \mathbb{E}_{Q_i(\bz_i)}[\phi_\al(\bx_i, \bz_i)] = \mathbb{E}_{R(\bx,\bz)}[\phi_\al(\bx, \bz)]
\end{equation}
We want to achieve high entropy for the distributions $\{Q_i\}$ and $R$, and meanwhile the average feature vector on the training set with hidden variables marginalized out should match the expected feature w.r.t. to the joint distribution of the model. The weights $\bw_\al$ act as Lagrange multipliers enforcing those constraints.

Similar to the derivation of herding for fully observed MRFs, we now introduce a temperature in Equation \ref{eqn:llh_pomrf} by replacing $\bw$ with $\bw/T$. Taking the limit $T\rightarrow 0$ of $\ell_T\defeq T\ell$, we see that the entropy terms vanish. For a given value of $\bw$ and in the absence of entropy, the optimal distribution $\{Q_i\}$ and $R$ are delta-peaks and their averages should be replace with maximizations, resulting in the objective,
\be
\ell_0(\bw) = \frac{1}{D}\sum_{i=1}^D \max_{\bz_i} \bw^T \bphi(\bx_i,\bz_i) - \max_{\bs}\bw^T \bphi(\bs)
\ee
where we denote $\bs=(\bx, \bz)$.

Taking a gradient descent update on $\ell_0$ with a fixed learning rate ($\eta=1$) defines the herding algorithm on POMRFs \citep{Welling09B}:
\begin{align}
\bz_{it}^* &= \arg\max_{\bz_i} \bw_{t-1}^T \bphi(\bx_i, \bz_i), \forall i \label{eqn:herding_pomrf_1}\\
\bs_t^* &= \arg\max_{\bs} \bw_{t-1}^T \bphi(\bs) \label{eqn:herding_pomrf_2}\\
\bw_t &= \bw_{t-1} + \left[\frac{1}{D}\sum_{i=1}^D \bphi(\bx_i, \bz_{it}^*)\right] - \bphi(\bs_t^*) \label{eqn:herding_pomrf_3}
\end{align}
We use a superscript ``$^*$" to denote states obtained by maximization. These equations are similar to herding for the fully observed case, but different in the sense that we need to impute the unobserved variables $\bz_i$ for every data-case separately through maximization. The weight update also consist of a positive ``driving term," which is now a changing average over data-cases, and a negative term, which is identical to the corresponding term in the fully observed case.

\subsubsection*{Moment Matching Property}

We can prove the boundedness of the weights with PCT by identifying\\ $\bv_t = \left[\frac{1}{D}\sum_{i=1}^D \bphi(\bx_i, \bz_{i,t+1}^*)\right] - \bphi(\bs_{t+1}^*)$, a finite set $V=\{\bv_t(\{\bz_i\},\bs)|\bz_i\in \mathcal{X}_\bz, \forall i, \bs\in\mathcal{X}\}$, and observing the inequality
\begin{align}
\bw_t^T\bv_t &= \left[\frac{1}{D}\sum_{i=1}^D \bw_t^T \bphi(\bx_i, \bz_{i,t+1}^*)\right] - \bw_t^T\bphi(\bs_{t+1}^*) \label{eqn:pct_pomrf} \\
&= \left[\frac{1}{D}\sum_{i=1}^D \max_{\bz_i}\bw_t^T \bphi(\bx_i, \bz_{i})\right] - \max_\bs \bw_t^T \bphi(\bs) \leq 0
\end{align}
The last inequality holds because the second term maximizes over more variables than the first term. Again, we do not have to be able to solve the difficult optimization problems of Equation \ref{eqn:herding_pomrf_1} and \ref{eqn:herding_pomrf_2}. Partial progress in the form of a few iterations of coordinate-wise descent is often enough to satisfy the condition in Equation \ref{eqn:pct_pomrf} which can be checked easily. 

Following a similar proof as Proposition \ref{prop:bounded_weight}, we obtain the fast moment matching property of herding on POMRFs:
\begin{proposition}\label{pro:bounded_weight_pomrf}
There exists a constant $R$ such that herding on a partially observed MRF satisfies
\be
\left|\frac{1}{\tau}\sum_{t=1}^\tau \frac{1}{D}\sum_{i=1}^D \phi_\al (\bx_i, \bz_{i t}^*) - \frac{1}{\tau}\sum_{t=1}^\tau \phi_\al(\bs_t^*) \right| \leq \frac{2R}{\tau}, \forall \al
\ee
\end{proposition}
Notice that besides a sequence of samples of the full state $\{\bs_t^*\}$ that form the joint distribution in the herding algorithm, we also obtain a sequence of samples of the hidden variables $\{\bz_{it}^*\}$ for every data case $\bx_i$ that forms the conditional distribution of $P(\bz_i|\bx_i)$. Those consistencies in the limit of $\tau\rightarrow \infty$ in Proposition \ref{pro:bounded_weight_pomrf} are in direct analogy to the maximum likelihood problem of Equation \ref{eqn:llh_pomrf} for which the following moment matching conditions hold at the MLE for all $\al$,
\be
\frac{1}{D}\sum_{i=1}^D\mathbb{E}_{P(\bz_i|\bx_i;\bw_{\textrm{MLE}})}[\phi_\al(\bx_i,\bz_i)] = \mathbb{E}_{P(\bx,\bz;\bw_{\textrm{MLE}})}[\phi_\al(\bx,\bz)]
\ee
These consistency conditions alone are not sufficient to guarantee a good model. After all, the dynamics could simply ignore the hidden variables by keeping them constant and still satisfy the matching conditions. In this case the hidden and visible subspaces completely decouple, defeating the purpose of using hidden variables in the first place. Note that the same holds for the MLE consistency conditions alone. However, an MLE solution also strives for high entropy in the hidden states. We observe in practice that the herding dynamics usually also induces high entropy in the distributions for $\bz$ avoiding the decoupling phenomenon described above.

The proof of the boundedness of weights depends on the assumption that we can find the global maximum in Equation \ref{eqn:herding_pomrf_2}, which is an intractable problem. \cite{Welling09B} also proposed a fully tractable herding variant that was guaranteed to satisfy PCT.

\begin{proposition}\label{pro:tractable_pomrf_herding}
Call $\mathcal{A}$ any tractable optimization algorithm to locate a local maximum in the product $\bw^T \bphi(\bx,\bz)$. This algorithm will be used to compute both $\bz_i^*$ and $\bs^*$. Call $\mathcal{E}_\mathcal{A}(\bx_i, \bw) = -\bw^T \bphi(\bx_i, \bz_i^*)$ the energy of data-case $i$ (note that this definition depends on the algorithm $\mathcal{A}$). Assume that given any initialization, $\mathcal{A}$ always return a state with an energy no larger than its initial state. Then the following tractable herding algorithm will remain in a compact region of weight space: Apply the usual herding updates with the difference that the optimization for $\bs^*$ is initialized at the state $(\bx_{i^*}, \bz_{i^*}^*)$ which represents the data-case with lowest energy $\mathcal{E}_\mathcal{A}(\bx_i, \bw)$.
\end{proposition}
\begin{proof}
The proof is trivial using the PCT condition as:
\begin{align}
& \bw_t^T\bv_t = -\left[\frac{1}{D}\sum_{i=1}^D \mathcal{E}_\mathcal{A}(\bx_i, \bw_t)\right] + \mathcal{E}_\mathcal{A}(\bs^*, \bw_t) \\
& \leq -\left[\frac{1}{D}\sum_{i=1}^D \mathcal{E}_\mathcal{A}(\bx_i, \bw_t)\right] + \mathcal{E}_\mathcal{A}(\bx_{i^*}, \bw_t) \leq 0
\end{align}
\end{proof}

\subsection{Herding Discriminative Models}\label{sec:ch}
We have been talking about running herding dynamics in an unsupervised learning setting. The idea of driving a nonlinear dynamical system to match moments can also be applied to discriminative learning by incorporating labels into the feature functions. Recalling the perceptron learning algorithm in \Sec{sec:pct}, the learning rule in Equation \ref{eqn:perceptron} can be reformulated in herding style:
\begin{align}
y_{i_t}^* &= \argmax_{y\in\{-1,1\}} \bw_{t-1}^T (\bx_{i_t}y) \\
\bw_t &= \bw_{t-1} + \bx_{i_t}y_{i_t} - \bx_{i_t}y_{i_t}^*
\end{align}
where we identify the feature functions as $\phi_j(\bx, y)=x_j y, j=1,\dots,m$, use mini-batches of size $1$ at every iteration, and do a partial maximization of the full state $(\bx, y)$ with the covariate $\bx$ clamped at the input $\bx_{i_t}$. The PCT guarantees that the moments (correlation between covariates and labels) $\mathbb{E}_{\mathcal{D}}[\bx y]$ from the training data are matched with $\mathbb{E}_{\mathcal{D}_\bx P(y^*|\bx)}[\bx y^*]$ where $p(y^*|x)$ is the model distribution implied by how the learning process generates $y^*$ with the sequence of weights $\bw_t$. The voted perceptron algorithm \citep{freund1999large} is an algorithm that runs exactly the same update procedure, applies the weights to make a prediction on the test data at every iteration $y_{\mathrm{test},t}^*$, and obtains the final prediction by averaging over iterations $y_{\mathrm{test}}^* = \mathrm{sign}(\frac{1}{\tau}\sum_{t=1}^\tau y_{\mathrm{test},t}^*)$. This amounts to learning and predicting based on the conditional expectation $\mathbb{E}_{P(y^*|\bx)}[y^*=1|\bx_{\mathrm{test}}]$ in the language of herding.

Let us now formulate the {\em conditional herding} algorithm in a more general way \citep{GelfandMaatenChenWelling10}. Denote the complete state of a data-case by $(\bx,\by,\bz)$ where $\bx$ is the visible input variable, $\by$ is the label, and $\bz$ is the hidden variable. Define a set of feature functions $\{\phi_\al(\bx, \by, \bz)\}$ with associated weights $\{w_\al\}$. Given a set of training data-cases, $\mathcal{D}=\{\bx_i, \by_i\}$, and a test set $\mathcal{D}_{\mathrm{test}}=\{\bx_{\mathrm{test},j}\}$, we run the conditional herding algorithm to learn the correlations between the inputs and the labels and make predictions at the same time using the following update equations:
\begin{align}
& \bz_{it}' = \argmax_{\bz_{i}}\bw_{t-1}^T\bphi(\bx_i,\by_i,\bz_i),\forall (\bx_i,\by_i)\in\mathcal{D} \label{eqn:herding_ch_1} \\
& (\by_{it}^*,\bz_{it}^*) = \argmax_{(\by_i,\bz_i)}\bw_{t-1}^T \bphi(\bx_i, \by_i,\bz_i), \forall \bx_i\in\mathcal{D}_\bx \label{eqn:herding_ch_2} \\
\bw_t &= \bw_{t-1} + \left[\frac{1}{D}\sum_{i=1}^D\bphi(\bx_i,\by_i,\bz_{it}')\right] - \left[\frac{1}{D}\sum_{i=1}^D\bphi(\bx_i,\by_{it}^*,\bz_{it}^*)\right] \label{eqn:herding_ch_3} \\
&(\by_{\mathrm{test},j,t}^*,\bz_{\mathrm{test},j,t}^*) = \arg\max_{(\by_j,\bz_j)}\bw_t^T \phi(\bx_{\mathrm{test},j}, \by_j,\bz_j), \forall \bx_{\mathrm{test},j}\in\mathcal{D}_{\mathrm{test}} \label{eqn:herding_ch_4}
\end{align}
In the positive term of Equation \ref{eqn:herding_ch_3}, we maximize over the hidden variables only, and in the negative term we maximize over both hidden variables and the labels. The last equation generates a sequence of labels, $\by_{\mathrm{test},j,t}^*$, that can be considered as samples from the conditional distribution of the test input from which we obtain an estimate of the underlying conditional distribution:
\be
P(\by|\bx_{\mathrm{test},j}) \approx \frac{1}{\tau}\sum_{t=1}^\tau \mathbb{I}(\by_{\mathrm{test},j,t}^* = \by) 
\ee

In general, herding systems perform better when we use normalized features: $\left\Vert \bphi(\bx,\bz,\by) \right\Vert = R,~\forall (\bx,\bz,\by)$. The reason is that herding selects states by maximizing the inner product $\bw^T\bphi$ and features with large norms will therefore become more likely to be selected. In fact, one can show that states inside the convex hull of the $\bphi(\bx,\by,\bz)$ are never selected. For binary ($\pm 1$) variables all states live on the convex hull, but this need not be true in general, especially when we use continuous attributes $\bx$. To remedy this, one can either normalize features or add one additional feature\footnote{If in test data this extra feature becomes imaginary we simply set it to zero.} $\bphi_0(\bx,\by,\bz)=\sqrt{R_\text{max}^2-||\bphi(\bx,\by,\bz)||^2}$, where $R_\text{max}=\mathop{\max}_{\bx,\by,\bz}\|\bphi(\bx,\by,\bz)\|$ with $\bx$ only allowed to vary over the data-cases.

We may want to use mini-batches $\mathcal{D}_t$ instead of the whole training  set for a more practical implementation, and the argument on the validity of using mini-batches in \Sec{sec:generalize_herding} applies here as well. It is easy to observe that Rosenblatts's perceptron learning algorithm is a special case of conditional herding when there are no hidden variables, $\by$ is a single binary variable, the feature function is $\bphi=\bx y$, and we use a mini-batch of size $1$ at every iteration.

Compared to the herding algorithm on partially observed MRFs, the main difference is that we do partial maximization in Equation \ref{eqn:herding_ch_2} with a clamped visible input $\bx$ on every training data-case instead of a joint maximization on the full state. Notice that in this particular variant of herding, the sequence of updates may converge when all the training data-cases are correctly predicted, that is, $\by_{it}^*=\by_{i}, \forall i=1,\dots,D$ at some $t$. For an example, the convergence is guaranteed to happen for the percepton learning algorithm on a linearly separable data set. We adopt the strategy in the voted perceptron algorithm \citep{freund1999large} which stops herding when convergence occurs and uses the sequence of weights up to that point for prediction in order to prevent the converged weights from dominating the averaged prediction on the test data.

Clamping the input variables allows us to achieve the following moment matching property:
\begin{proposition} \label{pro:bounded_weight_crf}
There exists a constant $R$ such that conditional herding with the update equations \ref{eqn:herding_ch_1}-\ref{eqn:herding_ch_3} satisfies
\be
\left|\frac{1}{D}\sum_{i=1}^D \frac{1}{\tau}\sum_{t=1}^\tau\phi_\al(\bx_i,\by_{it}^*,\bz_{it}^*) - \frac{1}{D}\sum_{i=1}^D \frac{1}{\tau}\sum_{t=1}^\tau \phi_\al (\bx_i, \by_i, \bz_{i t}') \right| \leq \frac{2R}{\tau}, \forall \al
\ee
\end{proposition}
The proof is straightforward by applying PCT where we identify
\be
\bv_t = \left[\frac{1}{D}\sum_{i=1}^D\bphi(\bx_i,\by_i,\bz_{it}')\right] - \left[\frac{1}{D}\sum_{i=1}^D\bphi(\bx_i,\by_{it}^*,\bz_{it}^*)\right],
\ee
the finite set $V = \{\bv(\{\bz_i'\},\{\by_i^*\},\{\bz_i^*\})|\bz_i'\in \mathcal{X}_\bz, \by_i^*\in\mathcal{X}_\by,\bz_i^*\in \mathcal{X}_\bz\}$, and observe the inequality $\bw_t^T\bv_t \leq 0$ because of the same reason as herding on POMRFs. Note that we require $V$ to be of a finite cardinality, which in return requires $\mathcal{X}_\by$ and $\mathcal{X}_\bz$ to be finite sets, but there is not any restriction on the domain of the visible input variables $\bx$. Therefore we can run conditional herding with input $\bx$ as continuous variables.

\subsubsection*{Zero Temperature Limit of CRF}

Consider a CRF with the probability distribution defined as
\be
P(\by,\bz|\bx;\bw)=\frac{1}{Z(\bw,\bx)}\exp\left(\sum_\al w_\al \phi_\al(\bx,\by,\bz)\right)
\ee
where $Z(\bw,\bx)$ is the partition function of the conditional distribution. The log-likelihood function for a dataset $\mathcal{D}=\{\bx_i,\by_i\}_{i=1}^D$ is expressed as
\be
\ell(\bw)=\frac{1}{D}\sum_{i=1}^D\left( \log\left(\sum_{\bz_i}\exp\left(\bw^T\bphi(\bx_i,\by_i,\bz_i\right)\right) - \log Z(\bw, \bx_i)\right)
\ee
Let us introduce the temperature $T$ by replacing $\bw$ with $\bw/T$ and take the limit $T\rightarrow 0$ of $\ell_T\defeq T\ell$. We then obtain the familiar piecewise linear Tipi function
\be
\ell_0(\bw) = \frac{1}{D}\sum_{i=1}^D\left(\max_{\bz_i}\bw^T\bphi(\bx_i,\by_i,\bz_i) - \max_{\by_i,\bz_i}\bw^T\bphi(\bx_i,\by_i,\bz_i) \right)
\ee
Running gradient descent updates on $\ell_0(\bw)$ immediately gives us the update equations of conditional herding \ref{eqn:herding_ch_1}-\ref{eqn:herding_ch_3}. 

Similar to the duality relationship between MLE on MRFs and the MaxEnt problem, MLE on CRFs is the dual problem of maximizing the entropy of the conditional distributions while enforcing the following constraints:
\be
\frac{1}{D}\sum_{i=1}^D\mathbb{E}_{P(\bz|\bx_i,\by_i)}\left[\phi_\al(\bx_i,\by_i,\bz)\right] = \frac{1}{D}\sum_{i=1}^D\mathbb{E}_{P(\by,\bz|\bx_i)}\left[\phi_\al(\bx_i,\by,\bz)\right] , \forall \al \label{eqn:herding_ch_constraints}
\ee
When we run conditional herding, those constraints are satisfied with the moment matching property in Proposition \ref{pro:bounded_weight_crf}, but how to encourage high entropy during the herding dynamics is again an open problem. We suggest some heuristics to achieve high entropy in the next experimental section. Note that there is a difference between MLE and conditional herding when making predictions. While the prediction of a CRF with MLE is made with the most probable label value at a point estimate of the parameters, conditional herding resorts to a majority voting strategy as in the voted perceptron algorithm. The regularization effect via averaging over predictions often provides more robust performance as shown later.

\section{Experiments}\label{sec:experiments}
We study the empirical performance of the herding algorithm introduced in \Sec{sec:property} and the extension with hidden variables in \Sec{sec:herding_pomrf} and for discriminative models in \Sec{sec:ch}.

\subsection{Herding with Fully Visible Models}
In the following experiments we will determine the ability of herding to convert information about the average value of features in the training data into estimates of some quantities of interest. In particular the input to herding will be joint probabilities of pairs of variables (denoted H.XX) and sometimes triples of variables (denoted H.XXX) where all variables will be binary valued (which is easily relaxed).

In experiment I we will consider the quantity $P(k)=\mathbb{E}[\eI [\sum_i X_i=k-1]]$ which is the distribution of the total number of $1$'s across all attributes. This quantity involves all variables in the problem and cannot be directly estimated from the input which consists of pairwise information only. This experiment measures the ability of herding to generalize from local information to global quantities of interest. In total $100$K samples were generated and used to estimate $P(k)$. The results were compared with the following two alternatives: 1) sampling $100$K pseudo-samples from the single variable marginals and using them to estimate $P(k)$ (denoted ``MARG''), 2) learning a fully connected, fully visible Boltzmann machine using the pseudo-likelihood method\footnote{This method is close to optimal for this type of problem \citep{parise2005learning}.} (denoted PL), then sampling $200$K samples from that model and using the last $100$K to estimate $P(k)$.

In experiment II we will estimate a discriminant function for classifying one attribute (the label) given the values of other attributes. Our approach was simply to perform online learning of a logistic regression function after each pseudo-sample collected from herding. Again, local pairwise information is turned into a global discriminant function which is then compared with some standard classifiers learned directly from the data. In particular, we compared against Naive Bayes, 5-nearest neighbors, logistic regression and a fully observed, fully connected Boltzmann machine learned with pseudo likelihood on the joint space of attributes and labels. The learned model's conditional distribution of label given the remaining attributes was subsequently used for prediction.

We have used the following datasets in our experiments.

A) The ``Bowling Data'' set\footnote{Downloadable from: \url{http://www.financialmathematics.com/wiki/Code:tenpin/data}}. Each binary attribute represents whether a pin has fallen during two subsequent bowls. There are $10$ pins and $298$ games in total. This data was generated by P. Cotton to make a point about the modelling of company default dependency. Random splits of $150$ train and $148$ test instances were used for the classification experiments.

B) Abalone dataset\footnote{Downloadable from UCI repository}. We converted the dataset into binary values by subtracting the mean from all (8) attributes and labels and setting all obtained values to 0 if smaller than 0 and 1 otherwise. For the classification task we used random subsets of $2000$ examples for training and the remaining $2177$ for testing.

C) ``Newsgroups-small"\footnote{Downloaded from: \url{http://www.cs.toronto.edu/~roweis/data.html}} prepared by S. Roweis. It has $100$ binary attributes and $16,242$ instances and is highly sparse ($4\%$ of the values is 1). Random splits of $10,000$ train and $6,242$ test instances were used for the classification experiments.

D) Digits: $8\times 8$ binarized handwritten digits. We used $1100$ examples from the digit classes 3 and 5 respectively (a total of $2200$ instances). The dataset contains $30\%$ 1's. This dataset was split randomly in $1600$ train and $600$ test instances.

The results for experiment I are shown in \tab{table:welling09A_1} and \Fig{fig:welling09A_1}. Note that the herding algorithms are deterministic and repetition would have resulted in the same values.

\begin{table}[t]
\begin{center}
\begin{tabular}{lllll}
\hline
\textsc{Dataset} & H.XXX & H.XX & PL & MARG \\
\hline
\textsc{Bowling} & 5E-3 & 4.1E-2 & 1.2E-1 & 4.3E-1 \\
\textsc{Abelone} & 8E-4 & 2.5E-3 & 2.2E-2 & 1.8E0 \\
\textsc{Digits} & - & 6.2E-2 & 3.3E-2 & 4E-1 \\
\textsc{News} & - & 2.5E-2 & 1.9E-2 & 5E-1 \\
\hline
\end{tabular}
\caption{\textbf{Abelone/Digits/NewsGroups}: KL divergence between true (data) distribution and the estimates from 1) herding algorithm using all triplets, 2) herding with all pairs, 3) samples from pseudo-likelihood model and 4) samples from single marginals.}
\label{table:welling09A_1}
\end{center}
\end{table}

\begin{figure}[tb]
  \centering
  \includegraphics[width=.6\textwidth]{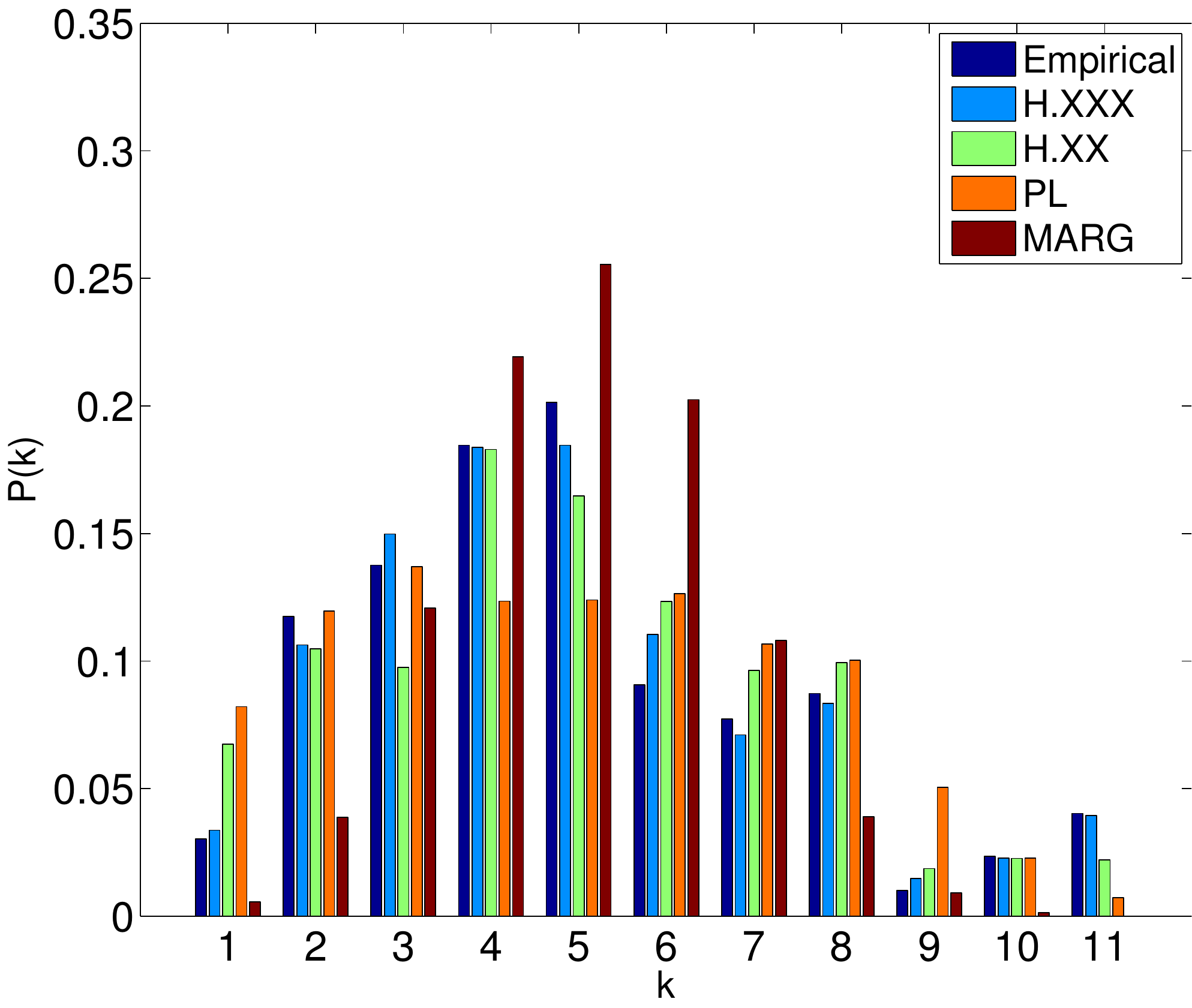}
  \caption{Estimates of $P(k)$ for the Bowling dataset. Each group of 5 bars represent the estimates for 1) ground truth, 2) herding with triples, 3) herding with pairs, 4) pseudo-likelihood, 5) marginals.}
  \label{fig:welling09A_1}
\end{figure}

We observe that herding is successful in turning local average statistics into estimates of global quantities. Providing more information such as joint probabilities over triplets does significantly improve the result (the triplet results for Digits and News took too long to run due to the large number of triplets involved). Also of interest is the fact that for the low dimensional data H.XX outperformed PL but for the high-D datasets the opposite was true while both methods seem to leverage the same second order statistics (even though PL needs the actual data to learn its model).

The results for the classification experiment are shown in \tab{table:welling09A_2}. On all tasks the online learning of a linear logistic regression classifier did just as well as running logistic regression on the original data directly. This implies that the herding algorithm generates the information necessary for classification and that the decision boundary can be learned online during herding. Interestingly, the PL procedure significantly outperformed all standard classifiers as well as herding on the Newsgroup data. This implies that a more sophisticated decision boundary is warranted for this data.

\begin{table}[tb]
\begin{center}
\begin{tabular}{l l l l l l}
\hline
\textsc{Dataset} & H.XXY & PL & 5NN & NB & LR \\
\hline
\textsc{Abelone} & $0.24\pm 0.004$ & $0.24\pm 0.004$ & $0.33\pm 0.1$ & $0.27\pm 0.006$ & $0.24\pm 0.004$ \\
\textsc{Bowling} & $0.23\pm 0.03$ & $0.28\pm 0.06$ & $0.32\pm 0.05$ & $0.23\pm 0.03$ & $0.23\pm 0.03$ \\ 
\textsc{Digits} & $0.05\pm 0.01$ & $0.06\pm 0.01$ & $0.05\pm 0.01$ & $0.09\pm 0.01$ & $0.06\pm 0.02$ \\ 
\textsc{News} & $0.11\pm 0.005$ & $0.04\pm 0.001$ & $0.13\pm 0.006$ & $0.12\pm 0.003$ & $0.11\pm 0.004$ \\ 
\hline
\end{tabular}
\caption{Average classification results averaged over 5 runs.}
\label{table:welling09A_2}
\end{center}
\end{table}

To see if the herding sequence contained the information necessary to estimate such a decision boundary we reran PL on the first 10,000 pseudo-samples generated by herding resulting in an error of 0.04, answering the question in the affirmative. A plot of the herding pseudo-samples as compared to the original data is shown in Figure 1.

\begin{figure}[tb]
  \centering
  \includegraphics[width=.6\textwidth]{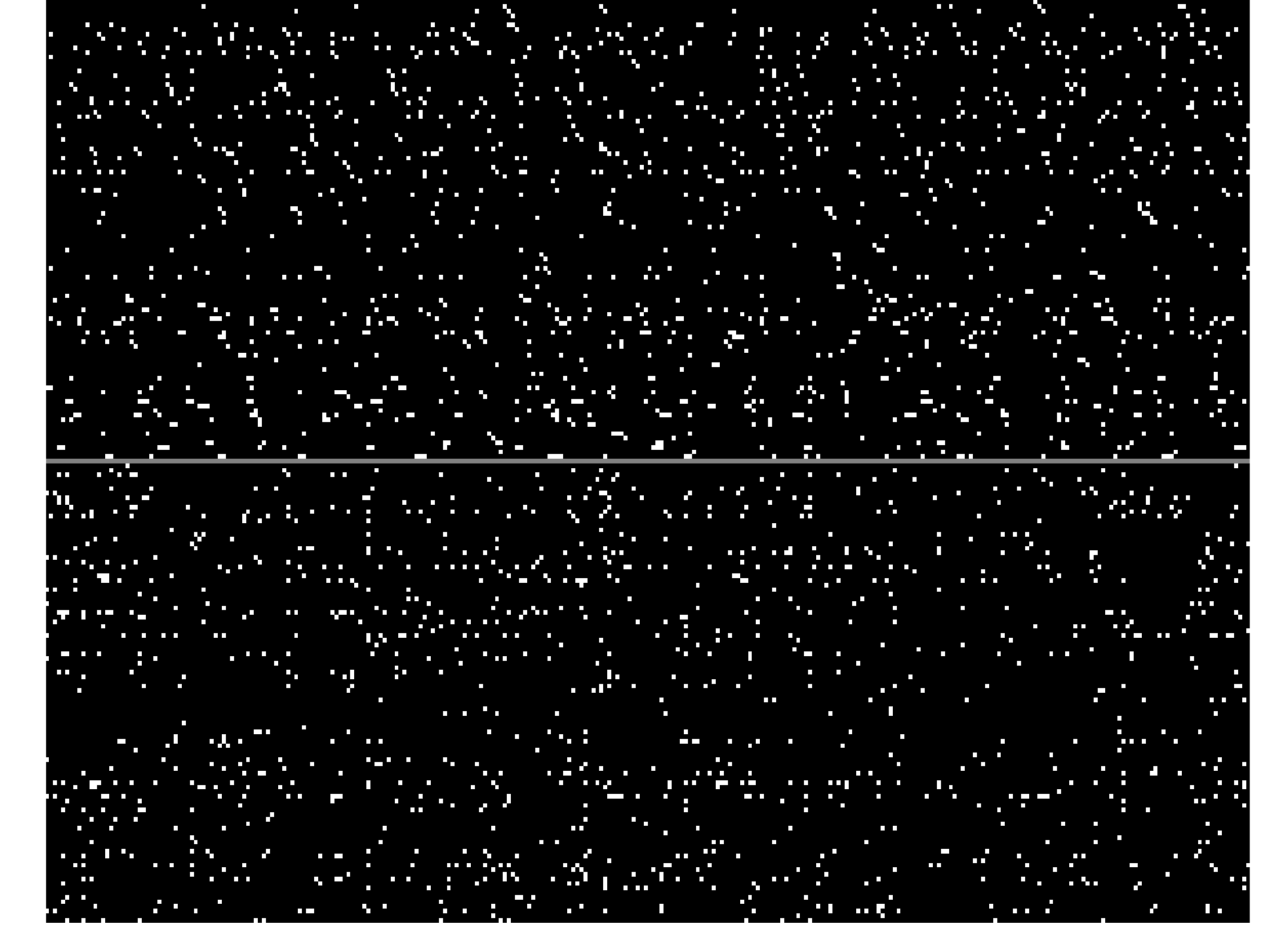}
  \caption{Top half: Sequence of 300 pseudo-samples generated from a herding algorithm for the ``Newsgroup" dataset. White dots indicate the presence of certain word-types in documents (represented as columns). Bottom half: Newsgroup data (in random order). Data and pseudo-samples have the same first and second order statistics.}
  \label{fig:welling09A_2}
\end{figure}

\subsection{Herding with Hidden Variables}
We studied generalized herding on the architecture of a restricted Boltzmann machine \citep{Hinton02}  (RBM). We used features $\bphi(x , z) = \{x_j, z_k, x_j z_k\}$, where $j$ and $k$ are indices of variables, and the $\{-1, +1\}$ representation because we found it worked significantly better than the $\{0, 1\}$ representation. To increase the entropy of the hidden units we left out the growth update for the features $\{z_k\}$ implying that $p(z_k = 1) \approx 0.5$. The intuition is the same as for bagging: we want to create a high diversity of (almost independent) ways to reconstruct the data because it will reduce the variance when making predictions. We observed that high entropy hidden representations automatically emerged when using a large number of hidden units. In contrast, for a small number of hidden units (say $K < 30$) there is a tendency for the system to converge on low entropy representations and the trick delivers some improvement.

We applied herding to the USPS Handwritten Digits dataset\footnote{Downloaded from \url{http://www.cs.toronto.edu/~roweis/data.html}} which consists of 1100 examples of each digit 0 through 9 (totaling $11,000$ examples). Each image has $256$ pixels and each pixel has a value between $[1..256]$ which we turned into a binary representation through the mapping $x_j' = 2\Theta(\frac{x_j}{256}-0.2)-1$ with $\Theta(x > 0) = 1$ and $0$ otherwise. Each digit class was randomly split into $700$ train, $300$ validation and $100$ test examples. As benchmarks we used 1NN using Manhattan distance and multinomial logistic regression, both in pixel space.

We used two versions of herding, one where the maximization over $\bs$ was initialized at the value from the previous time step (H) and one where we initialize at the data-case with the lowest energy (SH - the tractable algorithm). In both cases we ran herding for $2000$ iterations for each class individually. During the second $1000$ iterations we computed the energies for the training data in that class, as well as for all validation and test data across all classes. At each iteration we then used the training energies to standardize the validation and test energies by computing their Z-scores: $\cE_i' = (\cE_i - \mu_{\textrm{trn}})/\sigma_{\textrm{trn}}$ where $\mu_{\textrm{trn}}$ and $\sigma_{\textrm{trn}}$ represent the mean and standard deviation of the energies of the training data at that iteration. The standardized energies for test and validation data were subsequently averaged over herding iterations (using online averaging). Once we have collected these average standardized energies across all digit classes we fit a multinomial logistic regression classifier to the validation data, using the 10 class-specific energies as features.

We also compared these results against models learned with contrastive divergence \citep{Hinton02} (CD) and persistent CD \citep{Tieleman08} (PCD). For both CD and PCD we first applied (P)CD learning for $1000$ iterations in batch mode, using a stepsize of $\eta = 10^{-3}$. A momentum parameter of $0.9$ and 1-step reconstructions were used for CD. No momentum and a single sample in the negative phase was used for PCD. In the second $1000$ iterations we continued learning but also collected standardized validation and test energies as before which we subsequently used for classification. We have also experimented with chains of length $10$ and found that it did not improved the results but became prohibitively inefficient. To improve efficiency we experimented with learning in mini-batches but this degraded the results significantly, presumably because the number of training examples used to standardize the energy scores became less reliable.

\begin{figure}[tb]
  \centering
  \includegraphics[width=.6\textwidth]{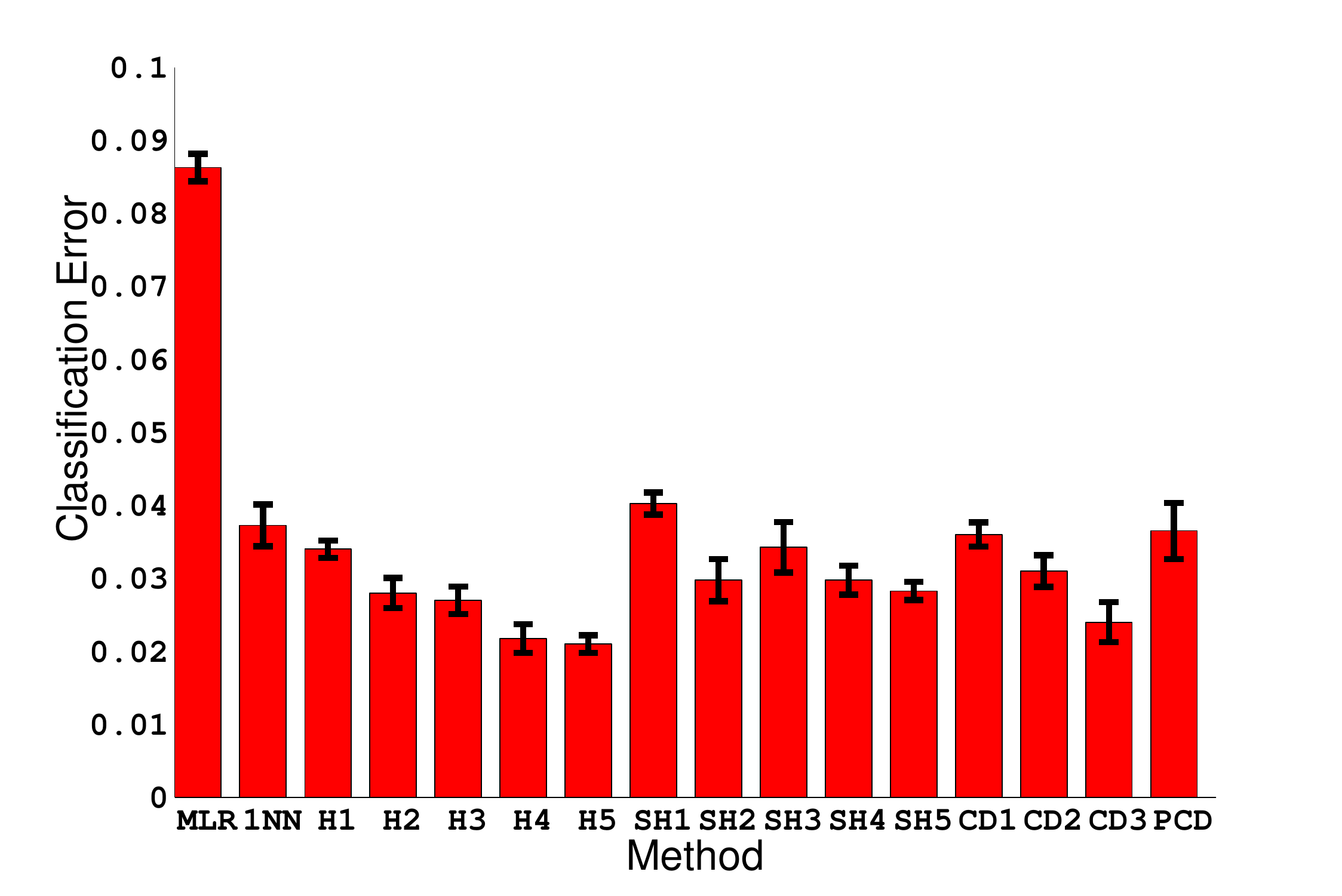}
  \caption{Classification results on USPS digits. 700 digits per class were used for training, 300 for validation and 100 for testing. Shown are average results over 4 different splits and their standard errors. From left to right: MLR (multinomial logistic regression), 1NN (1-nearest neighbor), H1-H5 (herding using local optimization with 50,100,250,500 and 1000 hidden units respectively), SH1-SH5 (safe, tractable herding from section 7 with 50,100,250,500 and 1000 hidden units respectively), CD1-CD3 (contrastive divergence with 50,100,250 hidden units respectively) and PCD (persistent CD with 500 hidden units).}
  \label{fig:welling09B_1}
\end{figure}

The results reported in \Fig{fig:welling09B_1} show the classification results averaged across 4 runs with different splits and for different values of hidden units. Without trying to claim superior performance we merely want to make the case that herding can be leveraged to achieve state-of-the-art performance (note that USPS error rates are higher than MNIST error rates). We also see that the tractable version of herding did not perform as well as the herding using local optimization, which in turn performed equally well as learning a model using CD. Persistent CD did not give very good results presumably because we did not use optimal settings for step-size, weight-decay etc.. It is finally interesting to observe that there does not seem to be any sign of over-fitting for herding. For the model with $1000$ hidden units, the total number of real parameters involved is around 1.5 million which represents more capacity than the 1.5 million binary pixel values in the data.

\subsection{Discriminative Herding}\label{sec:ch_experiments}
We studied the behavior of conditional herding on two artificial and four real-world data sets, comparing its performance to that of the voted perceptron \citep{freund1999large} and that of discriminative RBMs \citep{larochelle2008classification}. All the experiment results in this subsection are accredited to the authors of \cite{GelfandMaatenChenWelling10}.

\begin{figure}[tb]
  \centering
  \includegraphics[width=.5\textwidth]{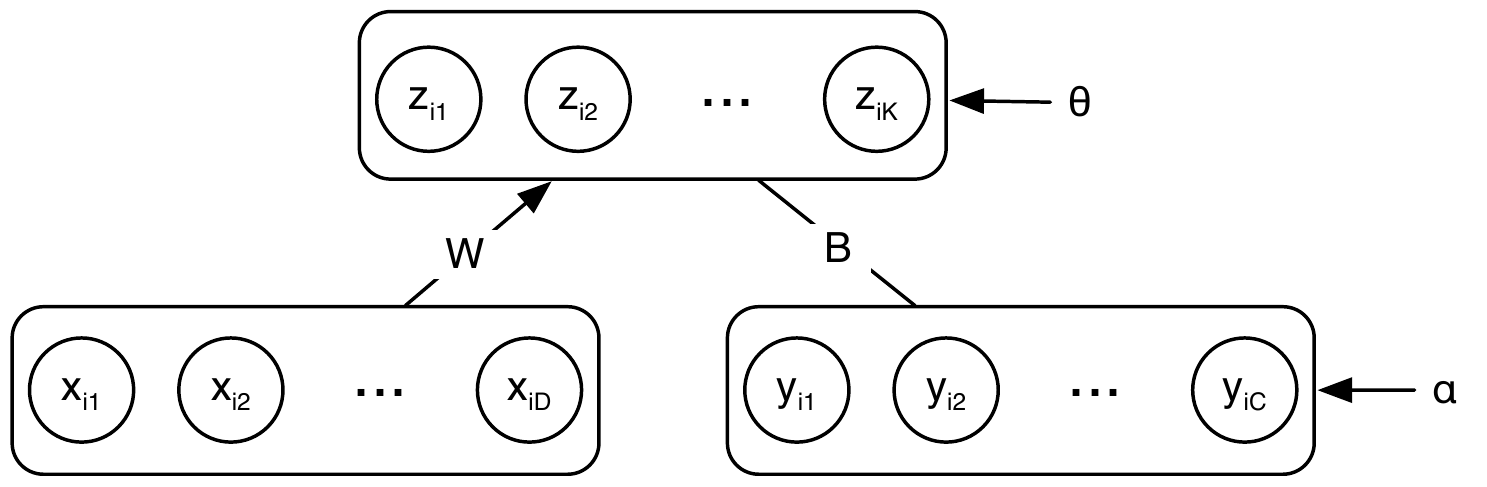}
  \caption{Discriminative Restricted Boltzmann Machine model of distribution $p(\mathbf{y},\mathbf{z} | \mathbf{x})$.}
  \label{fig:ch_RBM_model}
\end{figure}

We studied conditional herding in the discriminative RBM (dRBM) architecture illustrated in \Fig{fig:ch_RBM_model}, that is, we use the following parameterization
\be
\bw^T\bphi(\bx,\by,\bz)=\bx^T \bW \bz + \by^T \bB \bz + \bta^T\bz + \bal^T \by.
\label{eqn:ch_features_explicit}
\ee
where $\bW$, $\bB$, $\bta$ and $\bal$ are the weights, $\bz$ is a binary vector and $\by$ is a binary vector in a 1-of-$K$ scheme.

Per the discussion in \Sec{sec:ch}, we added an additional feature $\bphi_0(\bx)=\sqrt{R_\text{max}^2-||\bx||^2}$ with $R_{\max}=\max_{i}\left\Vert \mathbf{x}_{i}\right\Vert$ in all experiments.

\subsubsection*{Experiments on Artificial Data}
To investigate the characteristics of the voted perceptron (VP),  discriminative RBM (dRBM) and conditional herding (CH), we used the techniques discussed in \Sec{sec:ch} to construct decision boundaries on two artificial data sets: (1) the banana data set; and (2) the Lithuanian data set. We ran VP and CH for $1,000$ epochs using mini-batches of size $100$. The decision boundary for VP and CH is located at the location where the sign of the prediction $\mathbf{y}_{\text{test}}^*$ changes. We used conditional herders with $20$ hidden units. The dRBMs also had $20$ hidden units and were trained by running conjugate gradients until convergence. The weights of the dRBMs were initialized by sampling from a Gaussian distribution with a variance of $10^{-4}$. The decision boundary for the dRBMs is located at the point where both class posteriors are equal, i.e., where $p(y^*_{\text{test}}=-1 | \tilde{\mathbf{x}}_{\text{test}}) = p(y^*_{\text{test}}=+1 | \tilde{\mathbf{x}}_{\text{test}}) = 0.5$.

\begin{figure}[tbp!]
  \centering
  \subfloat[Banana data set.\label{fig:ch_banana_plot}] {
    \includegraphics[width=0.48\textwidth]{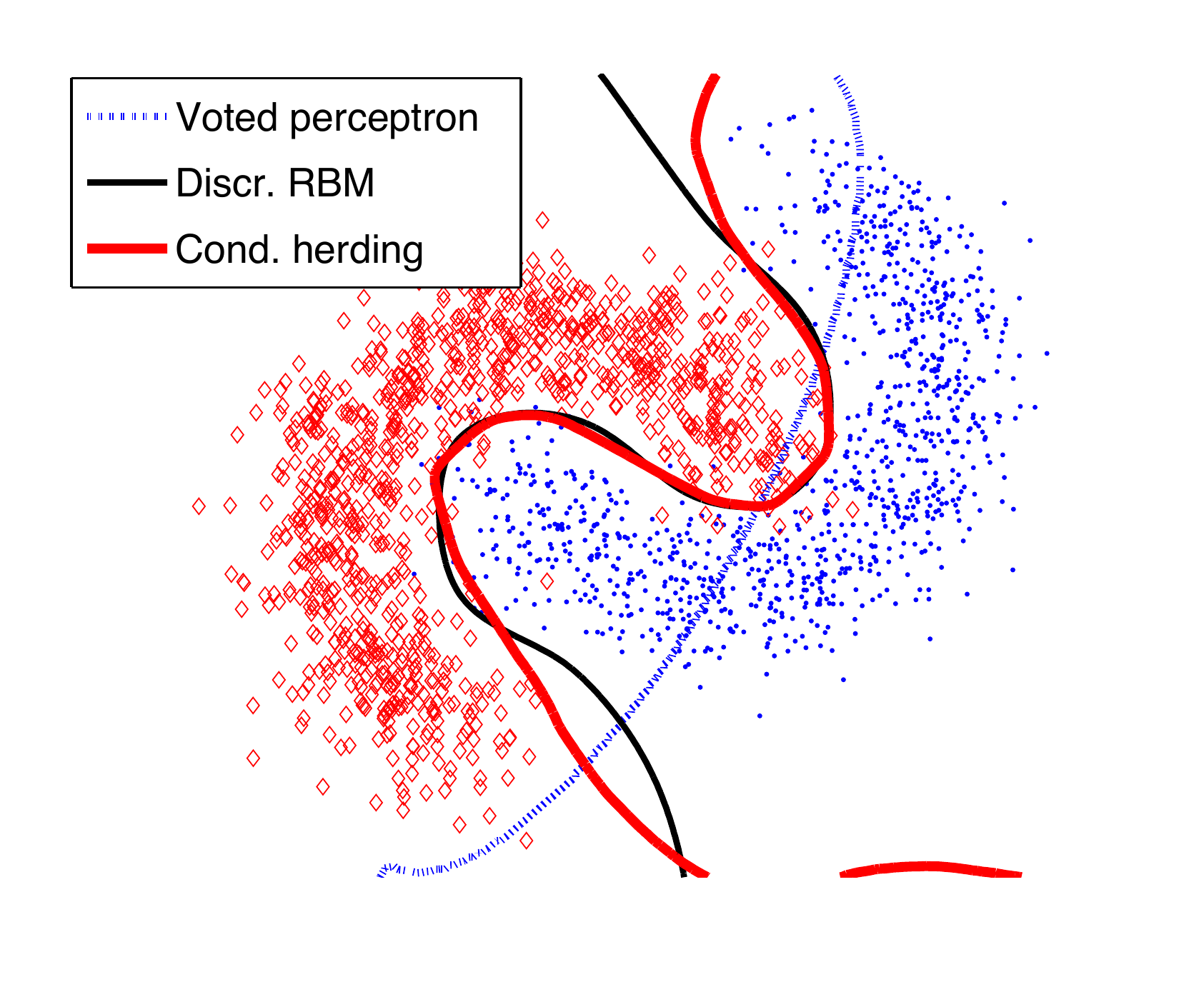}
  }
  \hfill
  \subfloat[Lithuanian data set.\label{fig:ch_lithuanian_plot}] {
    \includegraphics[width=0.48\textwidth]{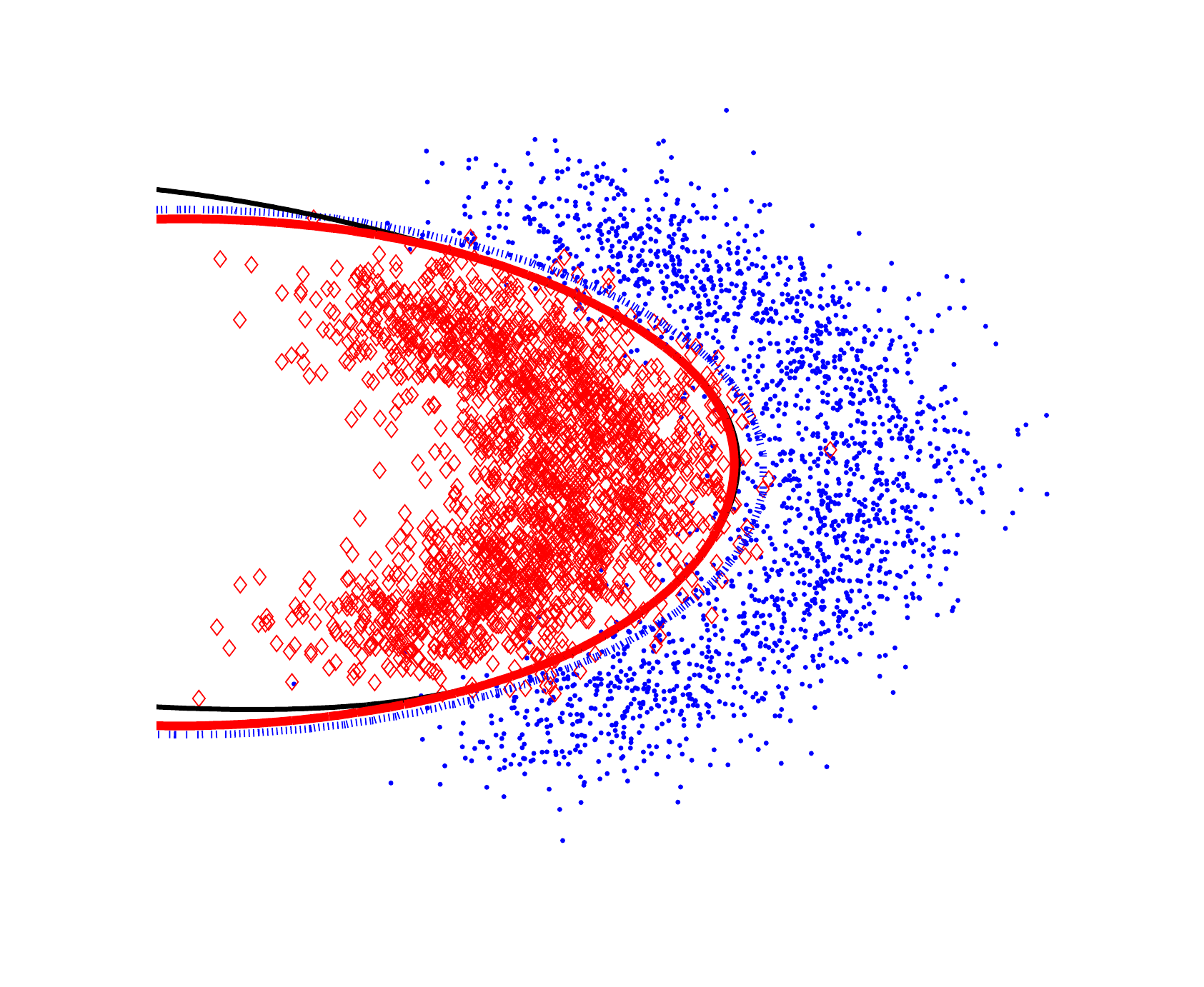}
  }
  \caption{Decision boundaries of VP, CH, and dRBMs on two artificial data sets.}
  \label{fig:ch_artificial_plots}
\end{figure}

Plots of the decision boundary for the artificial data sets are shown in \Fig{fig:ch_artificial_plots}. The results on the banana data set illustrate the representational advantages of hidden units. Since VP selects data points at random to update the weights, on the banana data set, the weight vector of VP tends to oscillate back and forth yielding a nearly linear decision boundary\footnote{On the Lithuanian data set, VP constructs a good boundary by exploiting the added `normalizing' feature.}. This happens because VP can regress on only $2+1=3$ fixed features. In contrast, for CH the simple predictor in the top layer can regress onto $M=20$ hidden features. This prevents the same oscillatory behavior from occurring.

\subsubsection*{Experiments on Real-World Data}
In addition to the experiments on synthetic data, we also performed experiments on four real-world data sets - namely, (1) the USPS data set, (2) the MNIST data set, (3) the UCI Pendigits data set, and (4) the 20-Newsgroups data set. The USPS data set consists of 11,000, $16\times16$ grayscale images of handwritten digits ($1,100$ images of each digit 0 through 9) with no fixed division. The MNIST data set contains $70,000$, $28\times28$ grayscale images of digits, with a fixed division into $60,000$ training and $10,000$ test instances. The UCI Pendigits consists of 16 (integer-valued) features extracted from the movement of a stylus. It contains $10,992$ instances, with a fixed division into $7,494$ training and $3,498$ test instances. The 20-Newsgroups data set contains bag-of-words representations of $18,774$ documents gathered from $20$ different newsgroups. Since the bag-of-words representation comprises of over $60,000$ words, we identified the $5,000$ most frequently occurring words. From this set, we created a data set of $4,900$ binary word-presence features by binarizing the word counts and removing the $100$ most frequently occurring words. The 20-Newsgroups data has a fixed division into $11,269$ training and $7,505$ test instances. On all data sets with real-valued input attributes we used the `normalizing' feature described above.

The data sets used in the experiments are multi-class. We adopted a 1-of-$K$ encoding, where if $\by_{i}$ is the label for data point $\bx_{i}$, then $\by_{i}=\left\{ y_{i,1},...,y_{i,K}\right\} $ is a binary vector such that $y_{i,k}=1$ if the label of the $i^{th}$ data point is $k$ and $y_{i,k}=-1$ otherwise. Performing the maximization in Equation \ref{eqn:herding_ch_2} is difficult when $K>2$. We investigated two different procedures for doing so. In the first procedure, we reduce the multi-class problem to a series of binary decision problems using a one-versus-all scheme. The prediction on a test point is taken as the label with the largest online average. In the second procedure, we make predictions on all $K$ labels jointly. To perform the maximization in Equation \ref{eqn:herding_ch_2}, we explore all states of $\by$ in a one-of-$K$ encoding - i.e. one unit is activated and all others are inactive. This partial maximization is not a problem as long as the ensuing configuration satisfies $\mathbf{w}_{t}^{T}\mathbf{v}_{t}\leq0$~\footnote{Local maxima can also be found by iterating over $y^{*,k}_{\text{test}},z^{*,k}_{\text{test},j}$, but the proposed procedure is more efficient.}. The main difference between the two procedures is that in the second procedure the weights $\bW$ are shared amongst the $K$ classifiers. The primary advantage of the latter procedure is its less computationally demanding than the one-versus-all scheme.

We trained the dRBMs by performing iterations of conjugate gradients (using $3$ line searches) on mini-batches of size $100$ until the error on a small held-out validation set started increasing (i.e., we employed early stopping) or until the negative conditional log-likelihood on the training data stopped coming down. Following~\cite{larochelle2008classification}, we use $L_2$-regularization on the weights of the dRBMs; the regularization parameter was determined based on the generalization error on the same held-out validation set. The weights of the dRBMs were initialized from a Gaussian distribution with variance of $10^{-4}$.

CH used mini-batches of size $100$. For the USPS and Pendigits data sets CH used a burn-in period of $1,000$ updates; on MNIST it was $5,000$ updates; and on 20 Newsgroups it was $20,000$ updates. Herding was stopped when the error on the training set became zero~\footnote{We use a fixed order of the mini-batches, so that if there are $D$ data cases and the batch size is $d$, if the training error is 0 for $\lceil D/d \rceil$ iterations, the error for the whole training set is $0$.}.

The parameters of the conditional herders were initialized by sampling from a Gaussian distribution. Ideally, we would like each of the terms in the energy function in Equation \ref{eqn:ch_features_explicit} to contribute equally during updating. However, since the dimension of the data is typically much greater than the number of classes, the dynamics of the conditional herding system will be largely driven by $\bW$. To negate this effect, we rescaled the standard deviation of the Gaussian by a factor $1/M$ with $M$ the total number of elements of the parameter involved (e.g. $\sg_\bW=\sg/(\dim(\bx)\dim(\bz))$ etc.). We also scale the learning rates $\boeta$ by the same factor so the updates will retain this scale during herding. The relative scale between $\boeta$ and $\sg$ was chosen by cross-validation. Recall that the absolute scale is unimportant (see \Sec{sec:ch} for details).

In addition, during the early stages of herding, we adapted the parameter update for the bias on the hidden units $\bta$ in such a way that the marginal distribution over the hidden units was nearly uniform. This has the advantage that it encourages high entropy in the hidden units, leading to more useful dynamics of the system. In practice, we update $\bta$ as $\bta_{t+1} = \bta_{t} + \frac{\boeta}{D_t} \sum_{i_t} (1-\lambda)\bra \bz_{i_t}\ket -  \bz^*_{i_t}$, where $i_t$ indexes the data points in the mini-batch at time $t$, $D_t$ is the size of the mini-batch, and $\bra\bz_{i_t}\ket$ is the batch mean. $\lambda$ is initialized to $1$ and we gradually half its value every $500$ updates, slowly moving from an entropy-encouraging update to the standard update for the biases of the hidden units.

VP was also run on mini-batches of size $100$ (with a learning rate of $1$). VP was run until the predictor started overfitting on a validation set. No burn-in was considered for VP.

The results of our experiments are shown in \tab{table:ch_results}. In the table, the best performance on each data set using each procedure is typeset in boldface. The results reveal that the addition of hidden units to the voted perceptron leads to significant improvements in terms of generalization error. Furthermore, the results of our experiments indicate that conditional herding performs on par with discriminative RBMs on the MNIST and USPS data sets and better on the 20 Newsgroups data set. The 20 Newsgroups data is high dimensional and sparse and both VP and CH appear to perform quite well in this regime. Techniques to promote sparsity in the hidden layer when training dRBMs exist (see~\cite{larochelle2008classification}), but we did not investigate them here. It is also worth noting that CH is rather resilient to overfitting. This is particularly evident in the low-dimensional UCI Pendigits data set, where the dRBMs start to badly overfit with $500$ hidden units, while the test error for CH remains level. This phenomenon is the benefit of averaging over many different predictors.

\begin{table}[t]
\begin{center}
\renewcommand{\tabcolsep}{0.1cm}
\begin{tabular}{|l||c||c|c||c|c|}\hline
\multicolumn{6}{c}{\textbf{One-Versus-All Procedure}} \\\hline
\multirow{2}{*}{\emph{Data Set}} & \textbf{VP} & \multicolumn{2}{c||}{\textbf{Discriminative RBM}} & \multicolumn{2}{c|}{\textbf{Conditional herding}}\\
 & & 100 & 200 & 100 & 200 \\\hline
\textbf{MNIST}         & \small{7.69\%} & \small{\textbf{3.57\%}} & \small{3.58\%} & \small{3.97\%} & \small{3.99\%} \\\hline
\multirow{2}{*}{\textbf{USPS}} & \small{5.03\%}   & \small{3.97\%} & \small{4.02\%} & \small{3.49\%} & \small{\textbf{3.35\%}} \\
                               & {\small(0.4\%)} & {\small(0.38\%)} & {\small(0.68\%)} & {\small(0.45\%)} & {\small(0.48\%)} \\\hline
\textbf{UCI Pendigits} & \small{10.92\%} & \small{5.32\%} & \small{5.00\%} & \small{3.37\%} & \small{\textbf{3.00\%}} \\\hline
\textbf{20 Newsgroups} & \small{27.75\%} & \small{34.78\%} & \small{34.36\%} & \small{29.78\%} & \small{\textbf{25.96\%}}\\\hline
\end{tabular}

\begin{tabular}{|l||c||c|c|c||c|c|c|}\hline
\multicolumn{8}{c}{\textbf{Joint Procedure}} \\\hline
\multirow{2}{*}{\emph{Data Set}} & \textbf{VP} & \multicolumn{3}{c||}{\textbf{Discriminative RBM}} & \multicolumn{3}{c|}{\textbf{Conditional herding}}\\
 & & 50 & 100 & 500 & 50 & 100 & 500 \\\hline
\textbf{MNIST} & \small{8.84\%} & \small{3.88\%} & \small{2.93\%} & \small{\textbf{1.98\%}} & \small{2.89\%} & \small{2.09\%} & \small{2.09\%} \\\hline
\multirow{2}{*}{\textbf{USPS}} & \small{4.86\%}   & \small{3.13\%} & \small{2.84\%} & \small{4.06\%} & \small{3.36\%} & \small{3.07\%} & \small{\textbf{2.81\%}} \\
                               & {\small(0.52\%)} & {\small(0.73\%)} & {\small(0.59\%)} & {\small(1.09\%)} & {\small(0.48\%)} & {\small(0.52\%)} & {\small(0.50\%)} \\\hline
\textbf{UCI Pendigits} & \small{6.78\%} & \small{3.80\%} & \small{3.23\%} & \small{8.89\%} & \small{3.14\%} & \small{\textbf{2.57\%}} & \small{2.86\%} \\\hline
\textbf{20 Newsgroups} & \small{\textbf{24.89\%}} & \small{ -- } & \small{30.57\%} & \small{30.07\%} & \small{ -- } & \small{25.76\%} & \small{24.93\%} \\\hline
\end{tabular}
\caption{Generalization errors of VP, dRBMs, and CH on 4 real-world data sets. dRBMs and CH results are shown for various numbers of hidden units. The best performance on each data set is typeset in boldface; missing values are shown as `-'. The std. dev. of the error on the $10$-fold cross validation of the USPS data set is reported in parentheses.}
\label{table:ch_results}
\end{center}
\end{table}

\section{Summary}\label{sec:summary}

We introduce the herding algorithm in this chapter as an alternative to the maximum likelihood estimation for Markov random fields. It skips the parameter estimation step and directly converts a set of moments from the training data into a sequence of model parameters accompanied by a sequence of pseudo-samples. By integrating the intractable training and testing steps in the regular machine learning paradigm, herding provides a more efficient way of learning and predicting in MRFs.

We study the statistical properties of herding and show that herding dynamics introduces negative auto-correlation in the sample sequence which helps to speed up the mixing rate of the sampler in the state space. Quantitatively, the negative auto-correlation leads to a fast convergence rate of $\cO(1/T)$ between the sampling statistics and the input moments. That is significantly faster than the rate of $\cO(1/\sqrt{T})$ that an ideal random sampler would obtain for an MRF at MLE. This distinctive property of herding should also be attributed to its weak-chaotic behavior as a deterministic dynamic system, whose characteristics deserve its own interest for future research.

Experiments confirms that the information contained in the pseudo-samples of herding can be used for inference and prediction. It achieves comparable performance with traditional machine learning algorithms including the MRFs, even though the sampling distribution of herding does not guarantee the maximum entropy.

We further provide a general condition, PCT, for the fast moment matching property. That condition allows more practical implementations of herding. We also use it to derive extensions of the herding algorithm for a wider range of applications. As more flexible feature functions defined on both visible and latent variables can now be handled in the generalized algorithm, we apply herding to training partially observed MRFs. Experiments on the USPS dataset show a classification accuracy on par with the state-of-art training algorithms on the same model. Furthermore, we propose a discriminative learning variant of herding for supervised problems by including labelling information in the feature definition. The resulting conditional herding provides an alternative to training CRFs.  Empirical evaluation shows competitive performance of herding compared with standard algorithms.

\section{Conclusion}\label{sec:conclusion}

The view espoused in this chapter is that we can view learning as an iterated map: $\bw_{t+1}=F(\bw_t)$ and that we can study the properties of this map using the tools of nonlinear dynamics systems. The usual learning approaches based on point estimates form a contractive map where all of parameter space is eventually mapped to a point. In Bayesian approaches we seek to find a posterior distribution over parameters and the map should thus converge to a distribution (or measure). For MCMC for instance the map consists of convolving the current distribution with a kernel. Herding offers a third possibility where the attractor is neither a point, nor a measure in the usual sense, but rather a highly complex, possibly fractal set. Interestingly, the more recent approach ``perturb and map''   is related to herding in the sense that it consists of a sequence of perturbations of the parameters followed by an optimization over the state space. However, it is different from herding in the sense the perturbations are generated randomly and IID, while in herding the perturbations are deterministic and dynamic (i.e. depend on the previous parameters).   

The surprising and powerful insight is that we can use a new set of tools from the mathematics literature to study these maps. For instance, it was shown in this chapter that herding dynamics is a special instance of the class of piecewise isometry maps, and should neither be classified as regular nor chaotic, but rather as what is known as ``edge of chaos". We suspect that this type of dynamics has useful properties in the context of learning from data. For instance, it seems related to the fact that the certain empirical moments averages exhibit very fast convergence. This is supported by the observations that 1) piecewise isometries have vanishing topological entropy, 2) exhibit the ``period doubling route to chaos'' and 3) have vanishing Lyapunov exponents. We believe that these type of concepts from the field of nonlinear dynamical systems may one day play an important role in the field of machine learning. 

\appendix{}

\subsection*{Some Results on Herding in Discrete Spaces}\label{app:discrete_variable}

The following proposition shows that the weight vectors move inside a $D-1$ dimensional subspace.
\begin{proposition}
For any herding dynamics with $D$ states and $K$ dimensional feature vectors, the trajectory of the weight vector lies in a subspace of a dimension $K^*\leq \max\{D-1, K\}$. Also, there exists an equivalent herding dynamics with $D$ states and $K^*$ dimensional feature vectors, which generates the same sequence of samples.
\end{proposition}
\begin{proof}
Let $\{\bphi(x_d)\}_{d=0}^{D-1}$ be the set of $D$ state feature vectors. Denote by $\Phi$ the subspace spanned of the set of $D-1$ vectors, $\{\bphi(x_d) - \bphi(x_0)\}_{d=1}^{D-1}$ in $\mathbb{R}^K$, and by $\Phi^\perp$ its complement. The dimension of $\Phi$ is apparently at most $\max\{D-1, K\}$. We want to construct a herding dynamics in $\Phi$ that generates the same sequence of states as the original dynamics.

Decompose the initial weight vector $\bw_0$ and all the feature vectors into $\Phi$ and $\Phi^{\perp}$, denoting the component in $\Phi$ with a superscript $^\parallel$ and in $\Phi^{\perp}$ with $^\perp$. Then $\bphi^{\perp} (x_d) = (\bphi(x_d) - \bphi(x_0) + \bphi (x_0))^{\perp} = \bphi^{\perp}(x_0), \forall d$ as $\bphi(x_d) - \bphi(x_0) \in \Phi$, and $\bphi^{\parallel} (x_d) = \bphi(x_d) - \bphi^{\perp}(x_0), \forall d$. Consequently $\bar{\bphi}^{\parallel} = \bar{\bphi} - \bphi^{\perp}(x_0)$ as $\bar{\bphi}$ is a convex combination of the feature vectors.

Let us consider a new herding dynamics (denoted by a superscript $^*$) with feature vectors $\{\bphi^{\parallel}(x_d)\}_{d=0}^{D-1}$ and the moment $\bar{\bphi}^{\parallel}$. We initialize with a weight vector $\bw_0^* = \bw_0^\parallel$. As $\Phi$ is closed with respect to the herding update in Equation \ref{eqn:herding_2} $\bw_t^*\in \Phi,\forall t\geq 0$. Now we want to show that the set of samples $S_T^*\defeq\{s_t^*\}_{t=1}^T$ is the same as $S_T\defeq\{s_t\}_{t=1}^T$ for any $T\geq 0$.

Obviously this holds at $T=0$ as $\bw_0^* \in \Phi$ and $S_T^*=S_T=\emptyset$. Assume that $S_T^*=S_T$ holds for some $T \geq 0$. Following the recursive representation of $\bw_T$ in Equation \ref{eqn:herding_2_expanded}, we get 
\begin{equation}
\bw^*_T =  \bw_0^* + T \bar\bphi^{\parallel} - \sum_{t=1}^T \bphi^{\parallel}(\bs_t)
=  \bw_0 - \bw_0^{\perp} + T \bar\bphi - \sum_{t=1}^T \bphi(\bs_t)
= \bw_T - \bw_0^\perp \label{eqn:equiv_herding_subspace}
\end{equation}
The sample to be generated at iteration $T+1$ is computed as
\begin{equation}
\bs^*_{T+1} = \arg\max_x (\bw^*_T)^T \bphi^\parallel(x) = \arg\max_x (\bw_T)^T \bphi(x) - (\bw_0^\perp)^T \bphi^{\perp}(x_0) = \bs_{T+1}
\end{equation}
Therefore, $S_{T+1}^*=S_{T+1}$, and consequently $S_T^*=S_T, \forall T\in [0,\infty)$ by induction. As a by-product of Equation \ref{eqn:equiv_herding_subspace}, we observe that the trajectory of the original herding dynamics $\{\bw_t\}$ lies in the $K^*$ dimensional affine subspace, $\bw_0^\perp + \Phi$.
\end{proof}

The proposition above suggests that the number of effective dimensions of the feature vector is upper-bounded by the number of states in the herding system. Also, the orthogonal component in the initial weight vector $\bw_0^\perp$ does not affect the sequence of generated samples. In our example of sampling a $D$-valued discrete distribution with the 1-of-D encoding, the $D$ feature vectors $\{\bphi(x_d)\}_{d=1}^{D-1}$ are linearly independent with each other and hence we achieve the maximum number of feature dimensions $K^*=D-1$. The affine subspace can be easily computed as $\{\bw : \sum_{d=1}^D w_d = 1 \}$. In the rest of this subsection, we will study the characteristics of a relatively more general type of herding dynamics with $D=K+1$ states, whose feature vectors consist of a linearly independent set in the $K$ dimensional feature space.

\begin{figure}[tb]
  \centering
  \includegraphics[width=.7\textwidth]{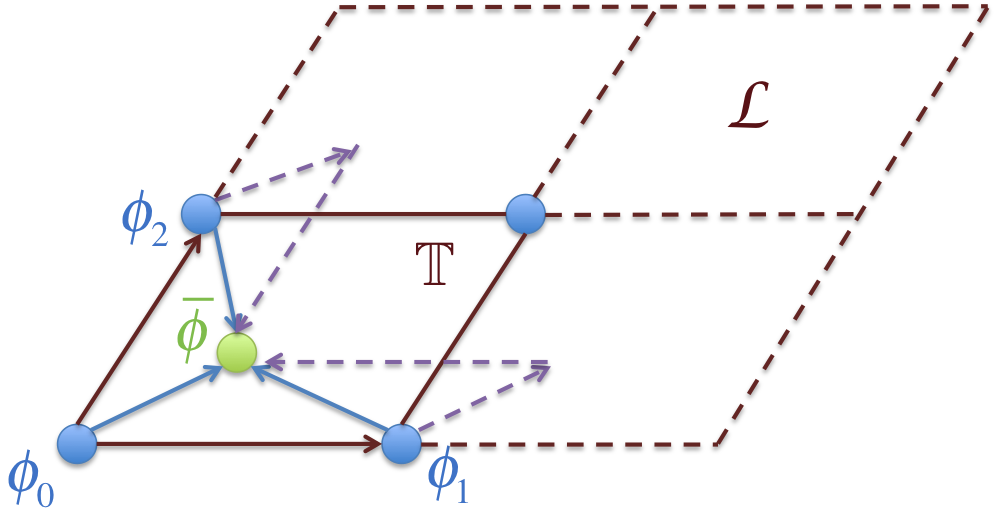}
  \caption{Example of the torus projection on herding dynamics with 3 states and 2-dimensional feature vectors. The red lines show the lattice and the torus (solid only) formed by $\bphi(x_1)-\bphi(x_0)$ and $\bphi(x_2)-\bphi(x_0)$, and the purple dashed arrows show that the herding dynamics corresponds to a constant rotation on the torus $\mathbb{T}^2$.}
  \label{fig:torus}
\end{figure}  
 
Let $\mathcal{L}$ be the lattice formed by the set of vectors $\{\bphi(x_d)-\bphi(x_0)\}_{d=1}^{K}$, and let $\mathbb{T}^K$ be the $K$ dimensional torus $\mathbb{R}^K/\mathcal{L}$. A torus is a circular space with every pair of opposite edges connected with each other. See \Fig{fig:torus} for an example of a 2D torus. Denote by $G:\mathbb{R}^K\rightarrow \mathbb{T}^K$ the canonical projection. For any point $u\in \mathbb{R}^K$, we have the property that $G(u + (\bphi(x_d)-\bphi(x_0))) = G(u), \forall d=0,\dots,K$. Let $\mathcal{T}:\mathbb{R}^K\rightarrow \mathbb{R}^K$ be the mapping of the herding dynamics in the feature space, which takes the form of a translation $\mathcal{T}(\bw) = \bw + \bar{\bphi} - \bphi(x(\bw))$, where $x(\bw)$ is the sample to be generated by Equation \ref{eqn:herding_1}. We can observe that the herding update on $\bw$ corresponds a rotation on the torus:
\begin{align}
G \circ \mathcal{T} (\bw) &= G (\bw + \bar{\bphi} - \bphi(x(\bw))) \nonumber\\
& = G (\bw + (\bar{\bphi} - \bphi(x_0)) - (\bphi(x(\bw)) - \bphi(x_0))) \nonumber\\
& = G(\bw) + (\bar{\bphi} - \bphi(x_0)) , \forall \bw \in \mathbb{R}^K
\end{align}
where the translation operator in $\mathbb{T}^K$ in the last equation refers to a rotation in the torus. This is an interesting property of herding with a maximum number of feature dimensions as it suggests that no matter what sample the dynamics takes, the trajectory of $\bw$ under the torus projection is driven by a constant rotation. Furthermore, if the set of elements in the translation vector $\bar{\bphi} - \bphi(x_0)$ is independent on rational numbers\footnote{Independence of a set of numbers, $x_1,\dots,x_K$, on rational numbers means that there does not exist a set of rational numbers $a_1, \dots, a_K$ that are not all zeros, such that $\sum_{d=1}^K a_d x_d = 0$.}, the trajectory on $\mathbb{T}^K$ fills the entire torus, which leads to a non-fractal attractor set with a finite volume in the original feature space.

{\small
\bibliographystyle{abbrvnat}
\bibliography{welling_chapter}

\begin{thebibliography}{37}
\providecommand{\natexlab}[1]{#1}
\providecommand{\url}[1]{\texttt{#1}}
\expandafter\ifx\csname urlstyle\endcsname\relax
  \providecommand{\doi}[1]{doi: #1}\else
  \providecommand{\doi}{doi: \begingroup \urlstyle{rm}\Url}\fi

\bibitem[Aihara and Matsumoto(1982)]{aihara1982temporally}
K.~Aihara and G.~Matsumoto.
\newblock Temporally coherent organization and instabilities in squid giant
  axons.
\newblock \emph{Journal of theoretical biology}, 95\penalty0 (4):\penalty0
  697--720, 1982.

\bibitem[Angel et~al.(2009)Angel, Holroyd, Martin, and
  Propp]{angel2009discrete}
O.~Angel, A.~E. Holroyd, J.~B. Martin, and J.~Propp.
\newblock Discrete low-discrepancy sequences.
\newblock \emph{arXiv preprint arXiv:0910.1077}, 2009.

\bibitem[Bach et~al.(2012)Bach, Lacoste-Julien, and Obozinski]{Bach2012herding}
F.~Bach, S.~Lacoste-Julien, and G.~Obozinski.
\newblock On the equivalence between herding and conditional gradient
  algorithms.
\newblock In J.~Langford and J.~Pineau, editors, \emph{Proceedings of the 29th
  International Conference on Machine Learning (ICML-12)}, ICML '12, pages
  1359--1366, New York, NY, USA, July 2012. Omnipress.
\newblock ISBN 978-1-4503-1285-1.

\bibitem[Bishop et~al.(2006)]{bishop2006pattern}
C.~M. Bishop et~al.
\newblock \emph{Pattern Recognition and Machine Learning}, volume~1.
\newblock springer New York, 2006.

\bibitem[Block and Levin(1970)]{block1970boundedness}
H.~Block and S.~Levin.
\newblock {On the boundedness of an iterative procedure for solving a system of
  linear inequalities}.
\newblock \emph{Proceedings of the American Mathematical Society}, 26\penalty0
  (2):\penalty0 229--235, 1970.

\bibitem[Bornn et~al.(2013)Bornn, Chen, de~Freitas, Eskelin, Fang, and
  Welling]{bornn2013herded}
L.~Bornn, Y.~Chen, N.~de~Freitas, M.~Eskelin, J.~Fang, and M.~Welling.
\newblock Herded {G}ibbs sampling.
\newblock In \emph{Proceedings of the International Conference on Learning
  Representations}, 2013.

\bibitem[Boshernitzan and Kornfeld(1995)]{boshernitzan1995interval}
M.~Boshernitzan and I.~Kornfeld.
\newblock Interval translation mappings.
\newblock \emph{Ergodic Theory and Dynamical Systems}, 15\penalty0
  (5):\penalty0 821--832, 1995.

\bibitem[Breuleux et~al.(2011)Breuleux, Bengio, and
  Vincent]{breuleux2011quickly}
O.~Breuleux, Y.~Bengio, and P.~Vincent.
\newblock Quickly generating representative samples from an rbm-derived
  process.
\newblock \emph{Neural Computation}, pages 1--16, 2011.

\bibitem[Chen et~al.(2010)Chen, Smola, and Welling]{ChenSmolaWelling10}
Y.~Chen, A.~Smola, and M.~Welling.
\newblock Super-samples from kernel herding.
\newblock In \emph{Proceedings of the Twenty-Sixth Conference Annual Conference
  on Uncertainty in Artificial Intelligence (UAI-10)}, pages 109--116,
  Corvallis, Oregon, 2010. AUAI Press.

\bibitem[Chen et~al.(2014)Chen, Gelfand, and
  Welling]{chen2014herdingbookchapter}
Y.~Chen, A.~E. Gelfand, and M.~Welling.
\newblock \emph{Advanced Structured Prediction}, chapter Herding for Structured
  Prediction, page 187.
\newblock The MIT Press, 2014.

\bibitem[Collins(2002)]{collins2002discriminative}
M.~Collins.
\newblock {Discriminative training methods for hidden markov models: Theory and
  experiments with perceptron algorithms}.
\newblock In \emph{Proceedings of the ACL-02 conference on Empirical methods in
  natural language processing-Volume 10}, page~8. Association for Computational
  Linguistics, 2002.

\bibitem[Freund and Schapire(1999)]{freund1999large}
Y.~Freund and R.~Schapire.
\newblock {Large margin classification using the perceptron algorithm}.
\newblock \emph{Machine learning}, 37\penalty0 (3):\penalty0 277--296, 1999.

\bibitem[Gelfand et~al.(2010)Gelfand, Chen, van~der Maaten, and
  Welling]{GelfandMaatenChenWelling10}
A.~Gelfand, Y.~Chen, L.~van~der Maaten, and M.~Welling.
\newblock On herding and the perceptron cycling theorem.
\newblock In J.~Lafferty, C.~K.~I. Williams, J.~Shawe-Taylor, R.~Zemel, and
  A.~Culotta, editors, \emph{Advances in Neural Information Processing Systems
  23}, pages 694--702, 2010.

\bibitem[Goetz(2000)]{goetz2000dynamics}
A.~Goetz.
\newblock Dynamics of piecewise isometries.
\newblock \emph{Illinois Journal of Mathematics}, 44\penalty0 (3):\penalty0
  465--478, 2000.

\bibitem[Harvey and Samadi(2014)]{harvey2014near}
N.~Harvey and S.~Samadi.
\newblock Near-optimal herding.
\newblock In \emph{Proceedings of The 27th Conference on Learning Theory},
  pages 1165--1182, 2014.

\bibitem[Hinton(2002)]{Hinton02}
G.~Hinton.
\newblock Training products of experts by minimizing contrastive divergence.
\newblock \emph{Neural Computation}, 14:\penalty0 1771--1800, 2002.

\bibitem[Huszar and Duvenaud(2012)]{Huszar12}
F.~Huszar and D.~Duvenaud.
\newblock Optimally-weighted herding is {B}ayesian quadrature.
\newblock In \emph{Proceedings of the Twenty-Eighth Conference Annual
  Conference on Uncertainty in Artificial Intelligence (UAI-12)}, pages
  377--386, Corvallis, Oregon, 2012. AUAI Press.

\bibitem[Larochelle and Bengio(2008)]{larochelle2008classification}
H.~Larochelle and Y.~Bengio.
\newblock {Classification using discriminative restricted {B}oltzmann
  machines}.
\newblock In \emph{Proceedings of the $25^{th}$ International Conference on
  Machine learning}, pages 536--543. ACM, 2008.

\bibitem[Lu and Wang(2005)]{LuWang05}
K.~Lu and J.~Wang.
\newblock Construction of {S}turmian sequences.
\newblock \emph{J. Phys. A: Math. Gen.}, 38:\penalty0 2891--2897, 2005.

\bibitem[Marston~Morse(1940)]{MorseHedlund40}
G.~A.~H. Marston~Morse.
\newblock Symbolic dynamics ii. sturmian trajectories.
\newblock \emph{American Journal of Mathematics}, 62\penalty0 (1):\penalty0
  1--42, 1940.
\newblock ISSN 00029327, 10806377.
\newblock URL \url{http://www.jstor.org/stable/2371431}.

\bibitem[Minsky and Papert(1969)]{minsky1969perceptrons}
M.~Minsky and S.~Papert.
\newblock \emph{Perceptrons: An Introduction to Computational Geometry}, volume
  1988.
\newblock MIT press Cambridge, MA, 1969.

\bibitem[Neal(1992)]{Neal92}
R.~Neal.
\newblock Connectionist learning of belief networks.
\newblock \emph{Articial Intelligence}, 56:\penalty0 71--113, 1992.

\bibitem[Neal(1993)]{neal1993probabilistic}
R.~Neal.
\newblock Probabilistic inference using {M}arkov chain {M}onte {C}arlo methods.
\newblock Technical Report CRG-TR-93-1, University of Toronto, Computer
  Science, 1993.

\bibitem[Papandreou and Yuille(2011)]{PaYu11a}
G.~Papandreou and A.~Yuille.
\newblock Perturb-and-map random fields: Using discrete optimization to learn
  and sample from energy models.
\newblock In \emph{Proc. {IEEE} Int. Conf. on Computer Vision (ICCV)}, pages
  193--200, Barcelona, Spain, Nov. 2011.
\newblock \doi{10.1109/ICCV.2011.6126242}.

\bibitem[Parise and Welling(2005)]{parise2005learning}
S.~Parise and M.~Welling.
\newblock Learning in {M}arkov random fields: An empirical study.
\newblock In \emph{Joint Statistical Meeting}, volume~4, page~7, 2005.

\bibitem[Rosenblatt(1958)]{rosenblatt1958perceptron}
F.~Rosenblatt.
\newblock {The perceptron: A probabilistic model for information storage and
  organization in the brain.}
\newblock \emph{Psychological Review}, 65\penalty0 (6):\penalty0 386--408,
  1958.

\bibitem[Salakhutdinov(2010)]{salakhutdinov27learning}
R.~Salakhutdinov.
\newblock Learning deep {Boltzmann} machines using adaptive {MCMC}.
\newblock In J.~F{\"u}rnkranz and T.~Joachims, editors, \emph{Proceedings of
  the 27th International Conference on Machine Learning (ICML-10)}, pages
  943--950, Haifa, Israel, June 2010. Omnipress.
\newblock URL \url{http://www.icml2010.org/papers/441.pdf}.

\bibitem[Swendsen and Wang(1987)]{swendsen1987nonuniversal}
R.~H. Swendsen and J.-S. Wang.
\newblock Nonuniversal critical dynamics in {M}onte {C}arlo simulations.
\newblock \emph{Physical Review Letters}, 58\penalty0 (2):\penalty0 80--88,
  1987.

\bibitem[Tieleman(2008)]{Tieleman08}
T.~Tieleman.
\newblock Training restricted {B}oltzmann machines using approximations to the
  likelihood gradient.
\newblock In \emph{Proceedings of the International Conference on Machine
  Learning}, volume~25, pages 1064--1071, 2008.

\bibitem[Tieleman and Hinton(2009)]{TielemanHinton09}
T.~Tieleman and G.~Hinton.
\newblock Using fast weights to improve persistent contrastive divergence.
\newblock In \emph{Proceedings of the International Conference on Machine
  Learning}, volume~26, pages 1064--1071, 2009.

\bibitem[Tsodyks et~al.(1098)Tsodyks, Pawelzik, and
  Markram]{TsodyksPawelzikMarkram98}
M.~Tsodyks, K.~Pawelzik, and H.~Markram.
\newblock Neural networks with dynamic synapses.
\newblock \emph{Neural Computation}, 10\penalty0 (4):\penalty0 821--835, 1098.

\bibitem[Welling(2009{\natexlab{a}})]{Welling09A}
M.~Welling.
\newblock Herding dynamical weights to learn.
\newblock In \emph{Proceedings of the 21st International Conference on Machine
  Learning}, Montreal, Quebec, CAN, 2009{\natexlab{a}}.

\bibitem[Welling(2009{\natexlab{b}})]{Welling09B}
M.~Welling.
\newblock Herding dynamic weights for partially observed random field models.
\newblock In \emph{Proceedings of the Twenty-Fifth Conference Annual Conference
  on Uncertainty in Artificial Intelligence (UAI-09)}, pages 599--606,
  Corvallis, Oregon, 2009{\natexlab{b}}. AUAI Press.

\bibitem[Welling and Chen(2010)]{WellingChen10}
M.~Welling and Y.~Chen.
\newblock Statistical inference using weak chaos and infinite memory.
\newblock In \emph{Proceedings of the Int'l Workshop on Statistical-Mechanical
  Informatics (IW-SMI 2010)}, pages 185--199, 2010.

\bibitem[Younes(1989)]{Younes89}
L.~Younes.
\newblock Parametric inference for imperfectly observed {G}ibbsian fields.
\newblock \emph{Probability Theory and Related Fields}, 82:\penalty0 625--645,
  1989.

\bibitem[Young(1954)]{young1954iterative}
D.~Young.
\newblock Iterative methods for solving partial difference equations of
  elliptic type.
\newblock \emph{Trans. Amer. Math. Soc}, 76\penalty0 (92):\penalty0 111, 1954.

\bibitem[Yuille(2004)]{Yuille04}
A.~Yuille.
\newblock The convergence of contrastive divergences.
\newblock In \emph{Advances in Neural Information Processing Systems},
  volume~17, pages 1593--1600, 2004.

\end{thebibliography}
}



\end{document}